\documentclass[11pt]{article}

\usepackage[preprint]{bags/acl}

\usepackage{times}
\usepackage{latexsym}

\usepackage[T1]{fontenc}
\usepackage[utf8]{inputenc}

\usepackage{microtype}

\usepackage{inconsolata}

\usepackage{graphicx}
\usepackage{booktabs}
\usepackage[table]{xcolor}
\usepackage{graphicx}
\usepackage{makecell}
\usepackage{multirow}
\usepackage{amsmath,amssymb}
\usepackage{algorithm}
\usepackage{algpseudocode}
\usepackage{float}
\usepackage{amsmath,amssymb,amsthm}
\newtheorem{lemma}{Lemma}
\newtheorem{theorem}{Theorem}
\usepackage{booktabs}
\usepackage{tabularx}
\usepackage{array}
\newcommand{\best}[1]{\textbf{#1}}
\newcommand{\second}[1]{\underline{#1}}

\definecolor{lightblue}{RGB}{235,244,251}

\usepackage{enumitem}
\setlist[itemize,1]{leftmargin=10pt, topsep=0pt, itemsep=0pt, parsep=0pt}

\newcommand{\responsebox}[2]{%
  \begin{minipage}[t]{0.485\linewidth}
    \vspace{0pt}
    \begingroup
    \setlength{\fboxsep}{6pt}
    \fcolorbox{black!35}{black!2}{%
      \begin{minipage}[t]
        {\dimexpr\linewidth-2\fboxsep-2\fboxrule\relax}
        \footnotesize\raggedright
        \setlength{\parindent}{0pt}
        \setlength{\parskip}{0pt}
        \textbf{#1}\par\medskip
        #2
      \end{minipage}%
    }
    \endgroup
  \end{minipage}%
}

\usepackage{enumitem}
\usepackage[most]{tcolorbox}

\usepackage{booktabs}
\usepackage{adjustbox}

\DeclareMathSizes{11}{9}{6}{5}
\DeclareMathSizes{10.95}{9}{6}{5}
\DeclareMathSizes{10}{9}{6}{5}
\DeclareMathSizes{9}{9}{6}{5}

\title{Learning to Reason with Selective Internal Feedback}

\author{
  \textbf{Xun Lu\textsuperscript{1,2}}\thanks{These authors contributed equally.},
  \textbf{Shumeng Yang\textsuperscript{1,2}}\footnotemark[1],
  \textbf{Xinrui Wang\textsuperscript{1,2}},
  \\
  \textbf{Songjun Tu\textsuperscript{1,2,3}},
  \textbf{Linjing Li\textsuperscript{1}}\thanks{Corresponding author.},
  \textbf{Daniel Dajun Zeng\textsuperscript{1}}
  \\[0.5em]
  \textsuperscript{1}State Key Laboratory of Multimodal Artificial Intelligence Systems,
  \\
  Institute of Automation, Chinese Academy of Sciences, Beijing, China
  \\
  \textsuperscript{2}School of Artificial Intelligence,
  \\
  University of Chinese Academy of Sciences, Beijing, China
  \\
  \textsuperscript{3}Pengcheng Laboratory
  \\[0.3em]
  \texttt{\{luxun2025,linjing.li\}@ia.ac.cn}
}

\begin{document}
\maketitle
\begin{abstract}
Reinforcement learning from internal feedback can improve language models without ground-truth answers or external verifiers. However, confidence and correctness can diverge during training, and continued optimization of average self-certainty can increase confidence while degrading reasoning accuracy. We introduce \textsc{SIF}, a reinforcement learning method that combines comparative reward construction with selective policy updates. \textsc{SIF} compares ranked token-level self-certainty profiles across multiple windows and peers, assigning positive rewards only when both levels of voting provide sufficient support. Sample and group level filtering determines which responses contribute to policy optimization. Experiments on three backbones across 12 mathematical and scientific reasoning benchmarks show that \textsc{SIF}'s final checkpoints achieve higher average accuracy than the best-validation checkpoints of Intuitor and Vigor, with a gain of 4.71 percentage points over Intuitor on Qwen2.5-Math-1.5B. Further analyses show better retention of accuracy during continued optimization, while test-time training experiments demonstrate sustained improvements with as few as four rollouts per prompt.
\end{abstract}

\section{Introduction}
\label{sec:introduction}

Reinforcement learning (RL) has substantially improved the reasoning capabilities of large language models (LLMs)~\cite{shao2024deepseekmathpushinglimitsmathematical,Guo_2025}. However, extending these gains remains constrained by the availability of suitable rewards. Verifiable rewards often depend on reference answers, task-specific verifiers, or execution environments~\cite{su2025crossingrewardbridgeexpanding}. Reinforcement learning from internal feedback (RLIF) reduces this dependence by deriving rewards from the model's own outputs and predictive distributions~\cite{zhang2025rightquestionhalfanswer}. Test-Time Reinforcement Learning (TTRL), for example, uses agreement among sampled answers, requiring enough responses with extractable final answers to estimate meaningful consensus~\cite{zuo2025ttrltesttimereinforcementlearning}. This motivates signals obtained directly from token distributions, such as self-certainty, which has shown promise for response selection~\cite{kang2025scalablebestofnselectionlarge}. Intuitor takes further by optimizing mean token-level self-certainty as its sole reward~\cite{zhao2026learningreasonexternalrewards}. Using a selection statistic as training reward, however, implicitly assumes its usefulness persists as the policy evolves. \emph{Does optimizing mean self-certainty lead to sustained reasoning improvements?}

\begin{figure}
    \centering
\includegraphics[width=\linewidth]{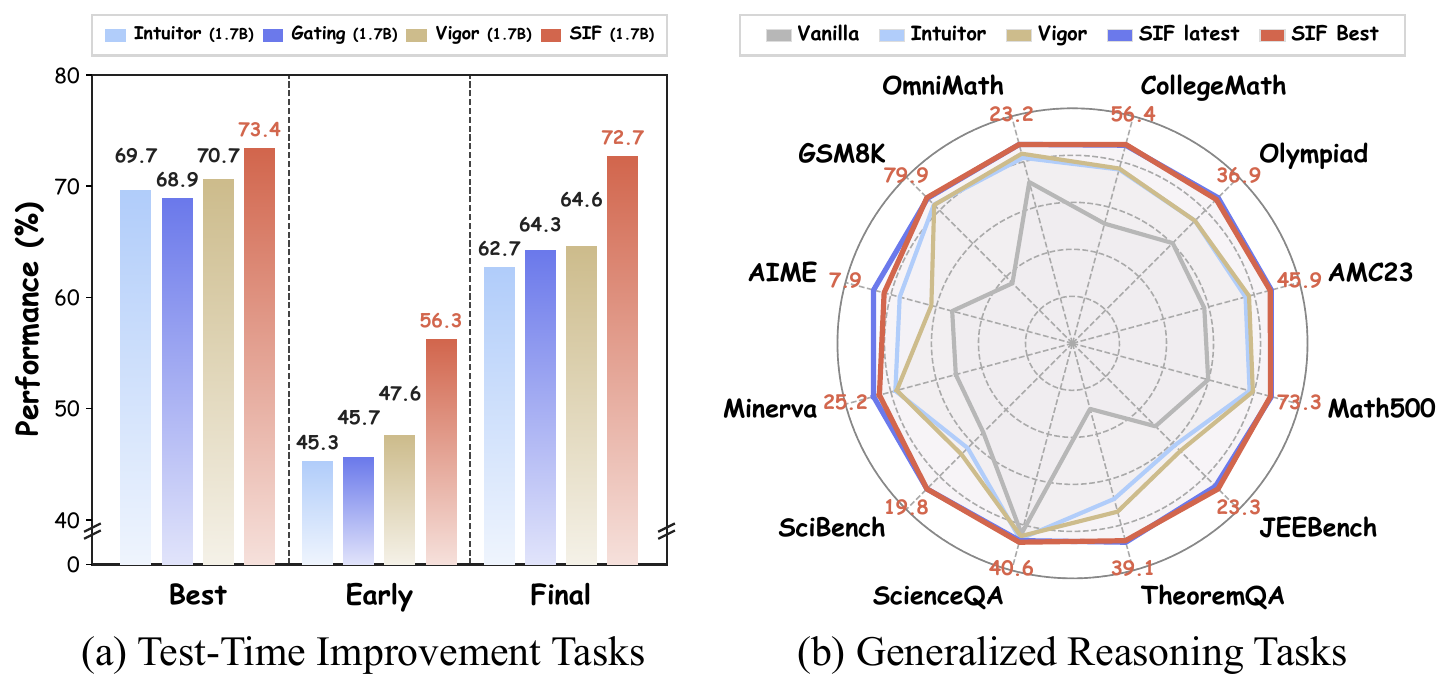}
\vspace{-19pt}
    \caption{Performance overview of SIF algorithm.}
    \vspace{-6pt}
\label{fig:Performance_overview}
\end{figure}

\begin{figure*}
    \centering
\includegraphics[width=\linewidth]{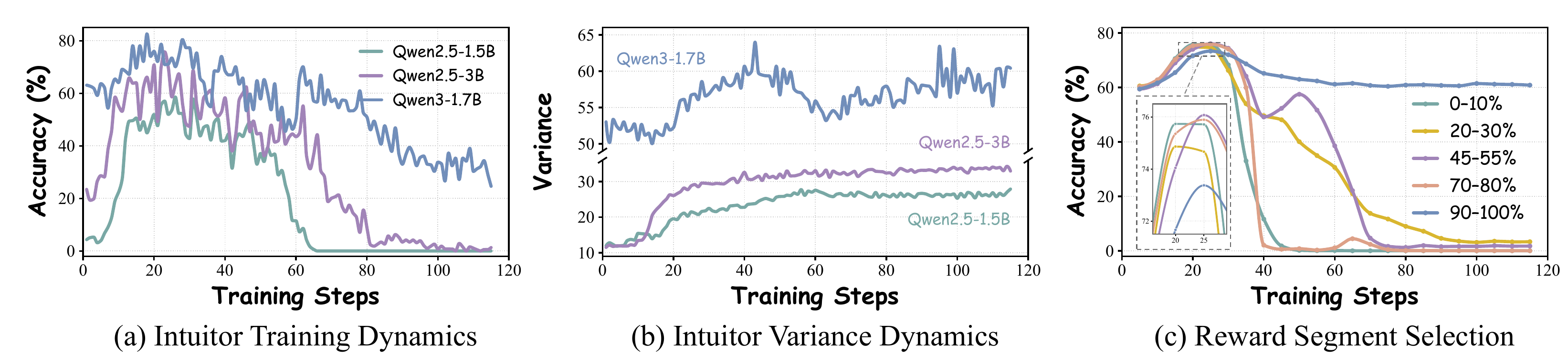}
\vspace{-17pt}
  \caption{
Training dynamics and diagnostic experiments for Intuitor.
(a) Training set accuracy across three backbones.
(b) Variance of token-level self-certainty during training.
(c) Validation accuracy of training variants using mean self-certainty within different rank intervals as rewards, with token scores sorted in descending order.
}
    \vspace{-1pt}
\label{fig:length_comparison}
\end{figure*}

This assumption does not always hold during training. High certainty in some segments can mask uncertainty elsewhere in sequence mean, and incorrect response may still receive a higher reward than a correct one. Policy updates can then reinforce these misleading preferences. Prior studies have reported early accuracy gains followed by degradation under continued optimization of internal rewards~\cite{zhang2025freelunchrethinkinginternal,ghimire2026prismunifiedframeworkposttraining,joarder2026betteronecollapsefreemultireward}. Self-certainty optimization can also lead to entropy collapse~\cite{jang-etal-2026-bad}. Figure~\ref{fig:training_dynamics} shows that Intuitor's training-batch accuracy initially improves but later declines substantially, accompanied by decreasing policy entropy. These observations raise a key question: \emph{How can we construct and selectively use self-certainty feedback to sustain performance gains during training?}

To address this question, we propose \textbf{S}elective \textbf{I}nternal \textbf{F}eedback (SIF), which builds comparative feedback from self-certainty profiles and selectively uses it for policy optimization. SIF assigns positive rewards only when comparisons across multiple ranked certainty levels and peer responses provide sufficient support. Response-level eligibility determines which responses participate in optimization, while groups with identical rewards are excluded from policy updates. All rewards are derived from model distributions, without final-answer extraction or answer-level pseudo-labels.

\begin{figure*}[t]
    \centering  \includegraphics[width=0.98\textwidth]{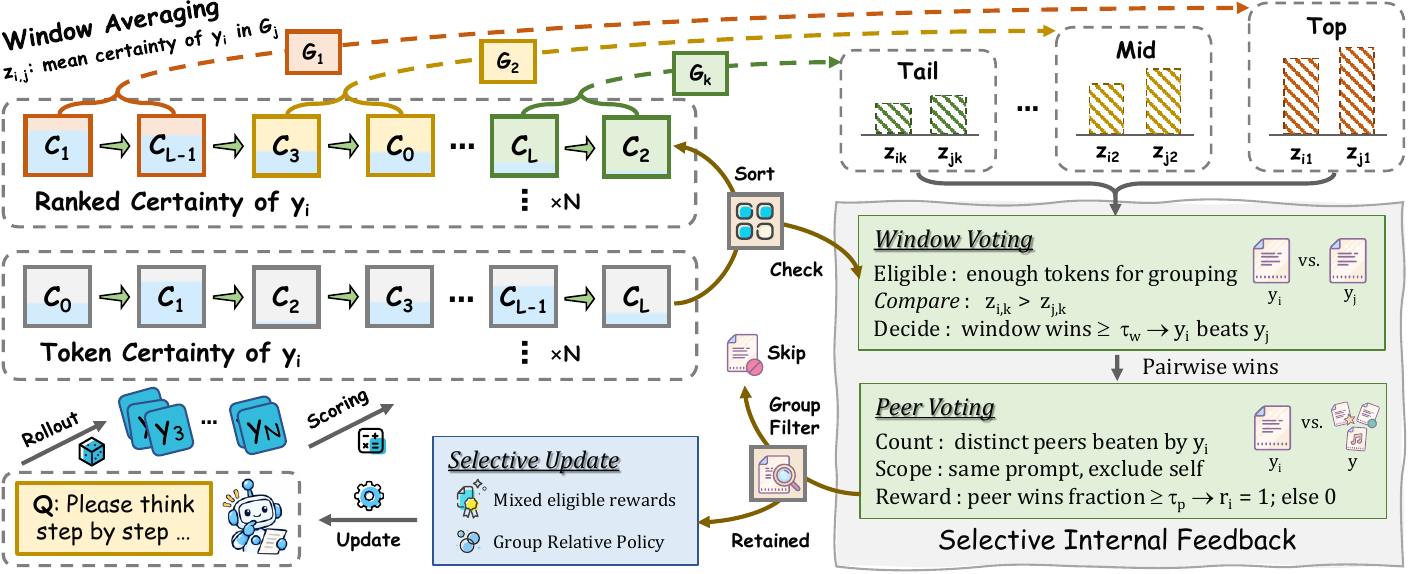}
   \caption{
Overview of SIF.
Token-level self-certainty scores are ranked and summarized across multiple windows. Window-level and peer-level voting jointly construct binary rewards, while response eligibility and group-level filtering determine participation in group-relative policy updates.
}
    \label{fig:overview}
     \vspace{-4pt}
\end{figure*}

We evaluate SIF through training on MATH with three backbones and evaluation across 12 mathematical and scientific reasoning benchmarks, as well as through test-time adaptation using unlabeled inputs. Following MATH training, SIF's final checkpoints achieve higher average accuracy than the best-validation checkpoints of Intuitor and Vigor~\cite{wen2026verifierfreerlllmsintrinsic} on all three backbones. Training diagnostics show reduced late-stage deterioration, while test-time adaptation experiments show improvements over the initial models persist to the final checkpoint. Figure~\ref{fig:Performance_overview} highlights SIF's performance gains over existing internal-feedback methods. The main contributions are as follows:

\begin{tcolorbox}[
    breakable,
    enhanced,
    colback=black!2,
    colframe=black!15,
    boxrule=0.1pt,
    arc=1mm,
    sharp corners=southwest,
    sharp corners=southeast,
    left=0.0mm,
    right=1.0mm,
    top=1.0mm,
    bottom=1.0mm
]
\begin{enumerate}[
    leftmargin=2em,
    itemsep=2pt,
    topsep=2pt,
    parsep=0pt
]
    \item We propose \textbf{two-level self-certainty voting}, which constructs binary internal rewards through comparisons across \emph{ranked certainty levels} and \emph{peer responses}.

\item We develop \textbf{selective policy updates} within SIF, integrating voting-based rewards with \textbf{response and group level filtering} under GRPO to jointly control reward construction and participation in policy optimization.

\item We evaluate SIF under \textbf{MATH training and test-time adaptation}. Across 12 benchmarks, its final checkpoints after MATH training achieve \textbf{higher average accuracy than the best-validation checkpoints of Intuitor and Vigor on all three backbones}, with gains of \textbf{4.71\%} and \textbf{4.31\%}, respectively, on Qwen2.5-Math-1.5B.

\end{enumerate}
\end{tcolorbox}

\section{Observation}
\label{sec:observation}

To better understand the performance degradation discussed in the introduction, we examine the dynamics of continued self-certainty optimization and the factors that affect performance retention. 

\noindent
\textbf{Notation and self-certainty.}
Let $\mathcal D$ be prompt dataset and $\pi_\theta$ trainable policy. For each prompt $x\in\mathcal D$, we sample a group of responses $y_i\sim\pi_\theta(\cdot\mid x)$, $i=1,\ldots,N$. Let $L_i$ denote valid token length of response $y_i$, and let $p_\theta(\cdot\mid x,y_{i,<t})$ denote predictive distribution at its $t$-th token position. We quantify self-certainty using KL divergence from uniform vocabulary distribution to predictive distribution:
\begin{equation}
\begin{aligned}
c_{i,t}
&=
D_{\mathrm{KL}}\!\left(
U \,\middle\|\,
p_\theta(\cdot\mid x,y_{i,<t})
\right)\\
&=
-\frac{1}{|\mathcal V|}
\sum_{v\in\mathcal V}
\log\!\left(
|\mathcal V|\,p_\theta(v\mid x,y_{i,<t})
\right),
\end{aligned}
\label{eq:obs-token-certainty}
\end{equation}
where $\mathcal V$ is the vocabulary and $U$ is the uniform distribution over $\mathcal V$. The mean self-certainty and variance within a response are
\begin{equation}
\begin{aligned}
\bar c_i
&=
\frac{1}{L_i}
\sum_{t=1}^{L_i}c_{i,t},\\
v_i
&=
\frac{1}{L_i}
\sum_{t=1}^{L_i}
(c_{i,t}-\bar c_i)^2.
\end{aligned}
\label{eq:sif-certainty}
\end{equation}
This variance is computed over valid tokens within each complete response, not across different responses to the same prompt.

\noindent
\textbf{(1) Early gains give way to deterioration.}
Intuitor uses mean token-level
self-certainty as its sole training reward.
Figure~\ref{fig:length_comparison}(a) shows that all three models achieve early accuracy gains, followed by varying degrees of deterioration as training continues.
Table~\ref{tab:diagnostic} further illustrates the divergence between
the internal reward and correctness: logged self-certainty increases
while MATH validation accuracy falls to zero in the reported runs
on Qwen3-1.7B and Llama-3.2-3B-Instruct.
These observations demonstrate that increasing self-certainty can coexist
with sharply declining accuracy, limiting its reliability as an objective
for sustained policy improvement. One possible explanation is that averaging over the entire response obscures differences within the certainty distribution. A preference based only on this mean may therefore lack consistent support across different parts of the response. To investigate this possibility, we examine the variance of self-certainty within responses.

\noindent
\textbf{(2) Certainty becomes more dispersed.}
Figure~\ref{fig:length_comparison}(b) shows that the variance within responses increases relative to the beginning of training across all three models, subsequently remaining elevated or fluctuating. This observation suggests that a single mean does not adequately characterize the certainty information within a complete response. Greater variance indicates a more dispersed distribution, so an increase in the overall mean therefore does not imply that all certainty levels strengthen together. These observations motivate a further question: do different certainty levels provide feedback with different training value?

\noindent
\textbf{(3) Certainty levels play distinct roles.}
For each response, we sort token-level self-certainty scores in descending
order and replace the mean over the entire response with the mean within a selected
ranked band. Each band defines a separate training variant, evaluated by validation accuracy in Figure~\ref{fig:length_comparison}(c). The figure~\ref{fig:length_comparison}(c) compares the
$0$--$10\%$, $20$--$30\%$, $45$--$55\%$, $70$--$80\%$, and $90$--$100\%$
bands, spanning the highest to the lowest certainty levels. 
These bands refer to positions in the certainty ranking, not contiguous spans of the original response.

The variants exhibit different training characteristics. The highest-certainty band yields rapid early gains but deteriorates substantially later; the intermediate band supports a high peak accuracy; and the lowest-certainty band retains more of its early improvement, ending at approximately $60\%$ validation
accuracy. These differences suggest that different certainty levels serve distinct roles during training. A single response mean obscures these distinctions, while selecting only one band leaves information from other levels unused.

\noindent
\textbf{Implication for SIF.}
These observations motivate comparative feedback
that requires support across multiple certainty
levels and multiple responses to the same prompt.
The following section introduces
SIF,
which implements this through voting
over self-certainty profiles and selective
policy updates.

\section{Method}
\label{sec:sif-method}
SIF constructs selective training feedback through pairwise comparisons of the model's self-certainty profiles. Its central operation tests whether a relative preference persists across multiple portions of a response's confidence ranking and across multiple peers. The procedure comprises four steps: sampling responses and computing self-certainty, constructing eligible ranked profiles, assigning two-level voting rewards, and selectively optimizing the policy. The policy reward is derived entirely from model distributions, without reference answers or an external verifier. Algorithm~\ref{alg:sif} provides the complete training procedure.

We follow the notation introduced in Section~\ref{sec:observation}. Let $\pi_{\mathrm{old}}$ denote the policy snapshot before the current rollout batch and $\pi_{\mathrm{ref}}$ the fixed reference policy. For each prompt $x$ in a batch $\mathcal B$ sampled from $\mathcal D$, we generate $N$ responses
$y_i\sim\pi_{\mathrm{old}}(\cdot\mid x)$, $i=1,\ldots,N$.
These responses form the comparison group from which SIF constructs its internal feedback.

\subsection{Response Eligibility and Ranked Profile}
\noindent \textbf {Length-based eligibility.}
A ranked profile requires enough valid tokens to support the configured windows. We use $K$ windows of width $w$, with $s$ specifying the minimum stride between window starts. The resulting response-length threshold, eligibility indicator, and eligible index set are
\begin{equation}
 \begin{aligned}
 L_{\min}&=w+(K-1)s,\\
 e_i&=\mathbf 1[L_i\geq L_{\min}],\\
 \mathcal E_x&=\{i:e_i=1\}.
 \end{aligned}
 \label{eq:sif-eligibility}
\end{equation}
The indicator $e_i$ determines whether response $i$ can participate in profile comparisons. Throughout, $\mathbf 1[\cdot]$ equals one when its condition holds and zero otherwise. The set $\mathcal E_x$ collects the eligible responses associated with prompt $x$. For $K>1$, the threshold provides sufficient space for $K$ windows whose starts are separated by at least $s$ positions.

\noindent \textbf {Constructing ranked profiles.}
For each eligible response, we sort its valid token-level self-certainties in descending order:
\begin{equation}
 c_{i,1}\geq c_{i,2}\geq\cdots
 \geq c_{i,L_i}.
 \label{eq:sif-sorting}
\end{equation}
Sorting makes the profile describe the distribution of self-certainty within a response. 

We place windows across this ranking to summarize different portions of the self-certainty distribution. The $k$-th window starts at
\begin{equation}
 a_{i,k}=
 \begin{cases}
 1, & K=1,\\
 1+\left\lfloor\dfrac{(k-1)(L_i-w)}{K-1}\right\rfloor,
    & K>1,
 \end{cases}
 \label{eq:sif-window-start}
\end{equation}
for $k=1,\ldots,K$. The floor operation maps each start to an integer position. For $K>1$, the starts span the ranking from its highest-certainty region to its lowest-certainty region. For $K=1$, the single window begins at the highest-ranked position.

Each window is summarized by its mean self-certainty, yielding the profile vector
\begin{equation}
 \begin{aligned}
 z_{i,k}&=\frac{1}{w}
   \sum_{u=a_{i,k}}^{a_{i,k}+w-1} c_{i,u},\\
 \mathbf z_i&=(z_{i,1},\ldots,z_{i,K}).
 \end{aligned}
 \label{eq:sif-profile}
\end{equation}
The scalar $z_{i,k}$ summarizes the $k$-th ranked window, and $\mathbf z_i$ contains all $K$ window summaries for response $i$. For $K>1$, the first and last windows contain the highest and lowest $w$ certainties, respectively. This supplies multiple comparison criteria while retaining the same profile dimension across responses of different eligible lengths.

\begin{figure}
    \centering
\includegraphics[width=\linewidth]{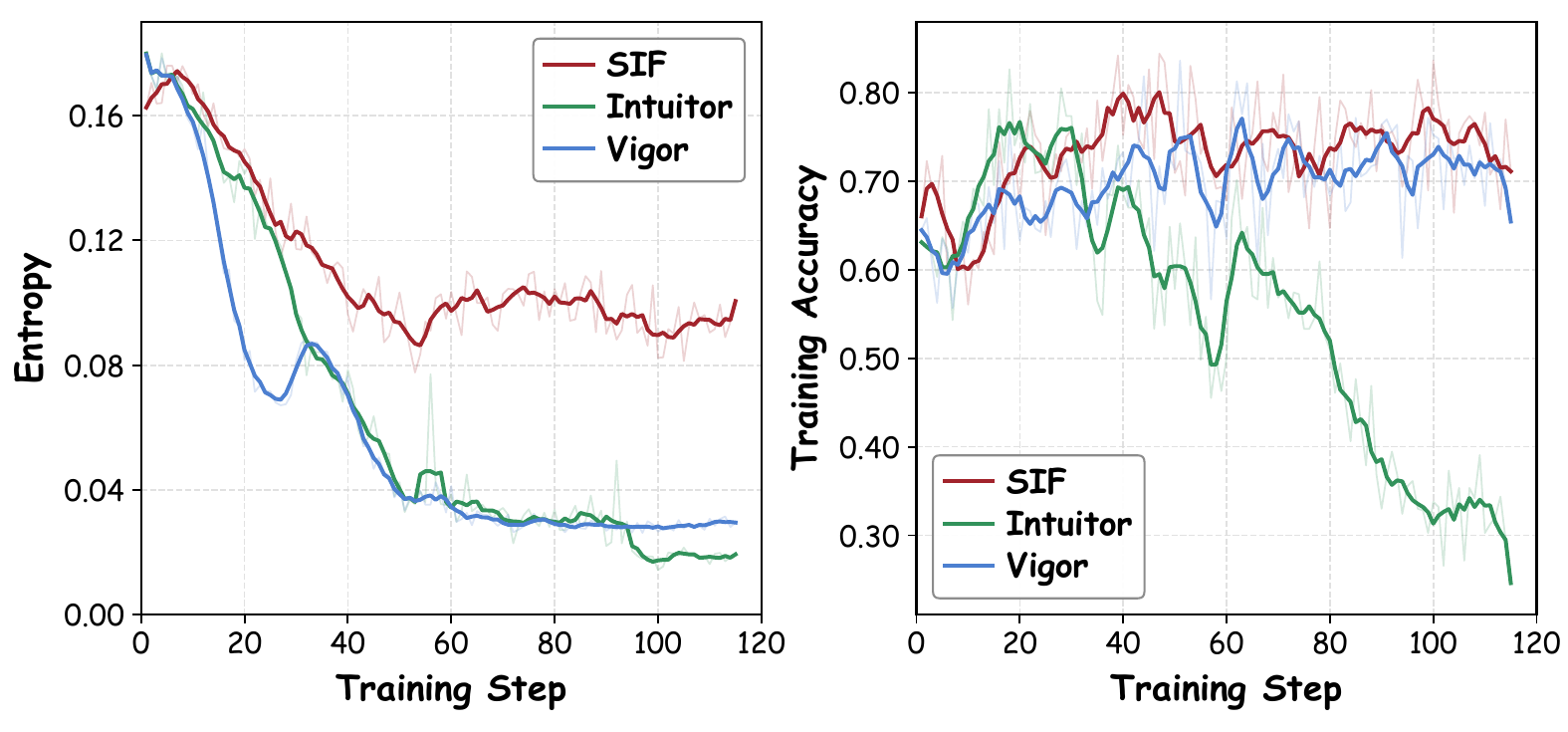}
\vspace{-19pt}
   \caption{Training dynamics of SIF, Intuitor, and Vigor on Qwen3-1.7B: policy entropy (left) and training-batch accuracy (right).}
    \vspace{-6pt}
\label{fig:training_dynamics}
\end{figure}

\begin{table*}[t]
\centering
\caption{
Accuracy (\%) on 12 reasoning benchmarks across three backbones.
All methods are trained and evaluated under identical settings within
each backbone. Baseline results use checkpoint achieving the highest
validation accuracy during training, while SIF reports final
(last) and best-validation (best) checkpoints.
Avg.\ denotes the unweighted mean across benchmarks.
The best and second-best distinct scores within each backbone,
excluding GRPO, are \textbf{bolded} and \underline{underlined}, with ties highlighted equally.
GRPO uses ground-truth (GT) answers.
}
\label{tab:main_results}
\setlength{\tabcolsep}{3.2pt}
\renewcommand{\arraystretch}{1.12}
\small
\resizebox{\textwidth}{!}{
\begin{tabular}{l*{13}{c}}
\toprule
\multirow{2}{*}{\textbf{Method}}
& \multicolumn{8}{c}{\textbf{Mathematical Reasoning}}
& \multicolumn{4}{c}{\textbf{Scientific Reasoning}}
& \multirow{2}{*}{\textbf{Avg.}} \\
\cmidrule(lr){2-9}
\cmidrule(lr){10-13}
& AIME
& AMC23
& Math500
& Minerva
& Olympiad
& GSM8K
& Omni
& College
& JEEBench
& Theorem
& Science
& SciBench
& \\
\midrule

% Llama-3.2-3B-Instruct
\textbf{Llama-3.2-3B-Instruct}
& \best{3.12} & 17.39 & 41.20 & 15.26
& 13.67 & 72.86 & 12.15 & 33.06
& 13.77 & 23.39 & 42.95 & 7.31
& 24.68 \\

\quad + TTRL
& 2.08 & 19.84 & \best{47.55} & \best{17.78}
& \best{16.34} & \best{78.91} & \second{13.72} & 33.91
& \second{14.60} & \best{26.41} & \second{44.01} & \best{9.64}
& \best{27.07} \\

\quad + SCOPE
& 1.67 & \second{20.38} & 44.33 & 16.82
& 14.61 & 74.85 & 13.02 & \second{34.69}
& 13.30 & 23.91 & 43.18 & 8.19
& 25.75 \\

\quad + Intuitor
& \second{2.71} & 17.93 & 41.48 & 16.18
& 14.73 & 72.34 & 12.34 & 33.92
& 13.76 & 22.57 & 43.13 & 7.83
& 24.91 \\

\quad + Vigor
& 1.25 & 11.14 & 39.30 & 10.06
& 12.22 & 67.67 & 9.74 & 26.32
& 9.94 & 20.05 & \best{44.38} & 7.16
& 21.60 \\

\rowcolor{lightblue}
\quad + SIF (last)
& \second{2.71} & 19.84 & 45.00 & 16.13
& 15.37 & \second{78.23} & 13.41 & 33.28
& 14.15 & 23.26 & 43.89 & 8.81
& 26.17 \\

\rowcolor{lightblue}
\quad + SIF (best)
& 1.67 & \best{21.74} & \second{46.73} & \second{17.65}
& \second{15.91} & 77.72 & \best{14.10} & \best{35.10}
& \best{14.72} & \second{24.41} & 43.99 & \second{8.92}
& \second{26.89} \\

\quad + GRPO (with GT)
& 6.05 & 19.84 & 47.02 & 17.92
& 15.97 & 79.23 & 14.51 & 35.59
& 14.84 & 26.49 & 43.91 & 9.28
& 27.55 \\

\midrule

% Qwen2.5-Math-1.5B
\textbf{Qwen2.5-Math-1.5B}
& 4.79 & 30.43 & 50.00 & 14.80
& 25.40 & 32.97 & 18.78 & 33.92
& 13.29 & 12.90 & 38.83 & 12.14
& 24.02 \\

\quad + TTRL
& \second{7.50} & 42.39 & 66.48 & 22.75
& 31.64 & 77.61 & \second{22.55} & 50.83
& 17.79 & 33.90 & 38.03 & 15.25
& 35.56 \\

\quad + SCOPE
& 5.83 & 41.58 & 65.35 & 22.79
& 31.45 & 77.24 & 22.12 & 50.57
& 19.22 & 35.24 & 38.35 & \second{15.47}
& 35.43 \\

\quad + Intuitor
& 6.88 & 39.95 & 65.25 & 22.47
& 31.06 & 76.10 & 21.61 & 49.26
& 16.36 & 30.59 & 39.92 & 14.27
& 34.48 \\

\quad + Vigor
& 5.62 & 40.76 & 66.45 & 22.29
& 31.05 & 75.92 & 22.09 & 49.53
& 17.26 & 33.10 & 39.32 & 15.12
& 34.88 \\

\rowcolor{lightblue}
\quad + SIF (last)
& \best{7.92} & \best{45.92} & \best{73.28} & \best{25.23}
& \best{36.91} & \second{79.60} & \best{23.17} & \second{56.14}
& \second{22.91} & \best{39.11} & \second{40.24} & \best{19.83}
& \best{39.19} \\

\rowcolor{lightblue}
\quad + SIF (best)
& \second{7.50} & \second{45.65} & \second{72.97} & \second{24.49}
& \second{36.46} & \best{79.92} & \best{23.17} & \best{56.36}
& \best{23.33} & \second{38.81} & \best{40.57} & \best{19.83}
& \second{39.09} \\

\quad + GRPO (with GT)
& 6.88 & 44.57 & 67.75 & 22.79
& 32.33 & 77.71 & 22.68 & 51.21
& 19.06 & 34.67 & 38.46 & 14.94
& 36.09 \\

\midrule

% Qwen3-1.7B
\textbf{Qwen3-1.7B}
& 10.00 & 33.97 & 68.47 & 32.67
& 32.86 & 88.09 & 18.50 & 51.16
& 26.89 & 40.86 & 42.57 & 24.11
& 39.18 \\

\quad + TTRL
& 12.08 & 45.65 & 76.90 & \second{35.25}
& 40.78 & \second{88.13} & 22.26 & \second{55.53}
& 31.37 & 46.17 & 42.67 & \second{27.84}
& 43.72 \\

\quad + SCOPE
& 12.08 & 46.74 & 77.08 & 34.83
& 41.39 & \best{88.37} & 22.26 & \best{55.74}
& 31.60 & 47.19 & \second{42.92} & 27.51
& 43.98 \\

\quad + Intuitor
& 15.21 & 47.83 & 75.88 & 30.19
& 44.34 & 82.93 & \second{26.78} & 50.67
& 36.04 & 47.82 & 41.57 & 22.29
& 43.46 \\

\quad + Vigor
& 10.62 & 45.11 & 74.52 & 30.93
& 41.65 & 82.12 & 24.90 & 53.49
& 33.16 & 43.66 & \best{43.64} & 19.03
& 41.90 \\

\rowcolor{lightblue}
\quad + SIF (last)
& \second{17.29} & \best{50.00} & \second{78.85} & 32.81
& \second{46.74} & 83.58 & \best{27.21} & 52.75
& \second{37.28} & \second{49.48} & 41.73 & 25.71
& \second{45.29} \\

\rowcolor{lightblue}
\quad + SIF (best)
& \best{18.54} & \second{49.73} & \best{79.28} & \best{35.29}
& \best{48.79} & 85.52 & 26.69 & 54.60
& \best{40.16} & \best{51.09} & 41.63 & \best{28.02}
& \best{46.61} \\

\quad + GRPO (with GT)
& 14.59 & 44.29 & 76.48 & 35.71
& 40.49 & 88.19 & 22.02 & 55.35
& 31.95 & 46.62 & 42.70 & 27.33
& 43.81 \\

\bottomrule
\end{tabular}
}
\end{table*}

\subsection{Two-Level Voting Reward}
\noindent \textbf {Window-level pairwise comparisons.}
SIF compares responses only within the same prompt group. For distinct eligible responses $i$ and $j$, it counts windows in which $i$ has higher mean self-certainty:
\begin{equation}
 v_{ij}=\sum_{k=1}^{K}\mathbf 1[z_{i,k}>z_{j,k}].
 \label{eq:sif-window-votes}
\end{equation}
The count $v_{ij}$ measures the support for preferring response $i$ over response $j$ across their ranked profiles. With window-win threshold $\tau_w$, we define the pairwise decision as
\begin{equation}
 b_{ij}=
 \begin{cases}
 \mathbf 1[v_{ij}\geq\tau_w],
   & i,j\in\mathcal E_x,\ i\neq j,\\
 0, & \text{otherwise}.
 \end{cases}
 \label{eq:sif-pairwise-wins}
\end{equation}
Threshold $\tau_w$ specifies how many window comparisons support pairwise win. Tied window means contribute to no vote, and self comparisons are excluded. $b_{ij}$ records whether profile of response $i$ receives sufficient support against response $j$.

\noindent \textbf {Peer-level reward assignment.}
A pairwise win provides support against one peer. We aggregate these decisions across the other eligible responses to obtain the peer-win count and binary reward:
\begin{equation}
\begin{aligned}
q_i &= \sum_{j\in\mathcal E_x\setminus\{i\}} b_{ij},\\
r_i &= e_i\,\mathbf{1}\!\left[
q_i \geq \tau_p\bigl(|\mathcal E_x|-1\bigr)
\right].
\end{aligned}
\end{equation}
Here, $q_i$ counts the other eligible peers defeated by
response $i$, and $\tau_p$ specifies the minimum peer-win
fraction required for a positive reward.
We set $\tau_p=3/4$ by default.
The two thresholds operate at distinct levels:
$\tau_w$ governs each pairwise comparison across ranked windows,
while $\tau_p$ governs the final response-level reward.
The reward reflects relative profile support within the
sampled group, rather than verified correctness.

\subsection{Selective Policy Optimization}
\noindent \textbf {Constructing group-relative advantages.}
We convert the binary response rewards into GRPO-style advantages. $q$, $o_i$, and $G$ denote the prompt, sampled responses, and rollout group size. The response reward and token-level advantage are
\begin{equation}
\begin{split}
r_i &= e_i\,\mathbf{1}\!\left[
q_i \geq \tau_p(|\mathcal E_x|-1)
\right], \\
\hat{A}_{i,t}
&= \frac{
r_i-\operatorname{mean}(\{r_j:j\in\mathcal E_x\})
}{
\operatorname{std}(\{r_j:j\in\mathcal E_x\})
+\epsilon_{\mathrm{num}}
},
\quad i\in\mathcal E_x.
\end{split}
\label{eq: advantage}
\end{equation}

\noindent \textbf {Updating the policy.}
For each query \(q\sim P(Q)\), where \(P(Q)\) is the training query distribution, the behavior policy \(\pi_{\mathrm{old}}=\pi_{\theta_{\mathrm{old}}}\) generates a group of \(G\) responses \(o_1,\ldots,o_G\).
SIF updates the target policy \(\pi_\theta\) using a clipped group-relative objective.
The importance ratio
\(\rho_{i,t}(\theta)=
\frac{\pi_\theta(o_{i,t}\mid q,o_{i,<t})}
{\pi_{\mathrm{old}}(o_{i,t}\mid q,o_{i,<t})}\)
compares the target and behavior policies, and
\(\operatorname{clip}_\varepsilon\) clips this ratio to
\([1-\varepsilon,1+\varepsilon]\).
The coefficient \(\beta\) controls the KL penalty relative to the fixed reference policy \(\pi_{\mathrm{ref}}\).
The resulting optimization objective is
\begin{equation}
\label{eq: grpo}
\begin{aligned}
&\mathcal{J}_{\mathrm{GRPO}}(\theta)
= \mathbb{E}\Bigg[
\frac{1}{G}\sum_{i=1}^{G}
\frac{1}{|o_i|}\sum_{t=1}^{|o_i|} \\
&\quad \bigg(
\min\Big[
\rho_{i,t}(\theta)\hat{A}_{i,t},\,
\operatorname{clip}_{\varepsilon}
(\rho_{i,t}(\theta))\hat{A}_{i,t}
\Big] \\
&\qquad
-\beta\mathbb{D}_{\mathrm{KL}}
(\pi_{\theta}\|\pi_{\mathrm{ref}})
\bigg)\Bigg].
\end{aligned}
\end{equation}
SIF applies the response- and group-level selection rules described above to determine which actor contributions enter policy optimization.

\begin{figure*}
    \centering
\includegraphics[width=\linewidth]{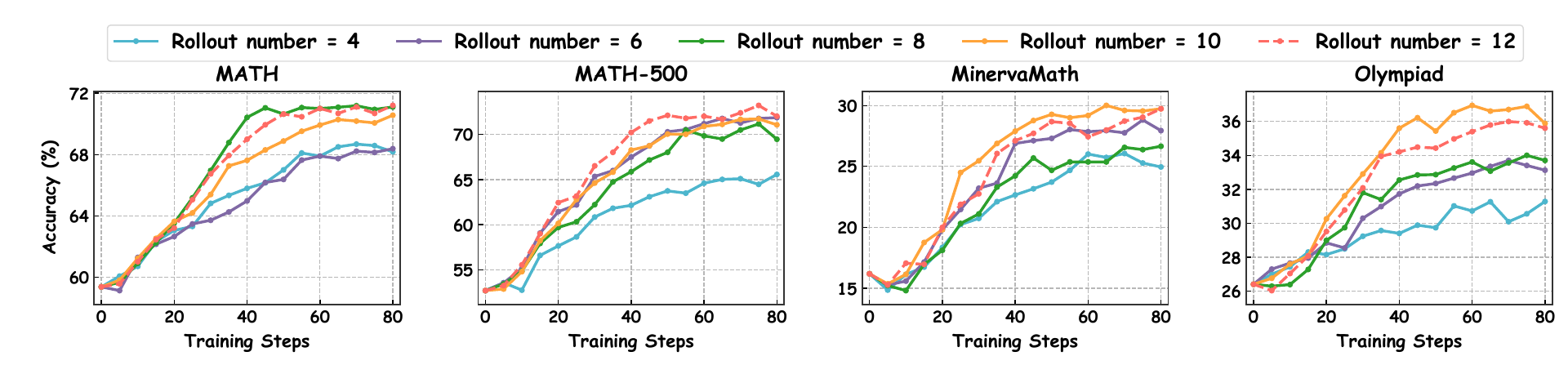}
\vspace{-20pt}
    \caption{Sustained improvements during Test-Time Training across four datasets with five rollout counts.}
    \vspace{-6pt}
\label{fig:rollout_comparison}
\end{figure*}

\section{Experiment}
\subsection{Training Setup}
\noindent \textbf {Baselines and Models.} We compare SIF with Intuitor, Vigor, TTRL, and SCOPE, alongside GRPO with ground-truth rewards. We consider two settings. (1) Validation Base: we train Llama-3.2-3B-Instruct~\cite{touvron2023llamaopenefficientfoundation}, Qwen2.5-Math-1.5B~\cite{yang2024qwen25mathtechnicalreportmathematical}, and Qwen3-1.7B~\cite{bai2023qwentechnicalreport} on MATH, select checkpoints using validation performance, and evaluate them on other reasoning benchmarks. (2) Test-time training: we adapt Qwen2.5-Math-1.5B and DeepSeek-R1-Distill-Qwen-1.5B~\cite{deepseekai2025deepseekr1incentivizingreasoningcapability} to unlabeled benchmark inputs for 80 steps. Checkpoints are evaluated at predefined steps, without using reference answers for optimization or checkpoint selection. Implementation and evaluation details are provided in Appendices~\ref{app:math_training}, \ref{app:test_time_config}, and~\ref{app:baseline_configs}.

\noindent \textbf {Evaluation Benchmarks.} We evaluate on 12 benchmarks covering two domains: mathematical reasoning on AIME~\cite{aime25,aime26}, AMC23, MATH500~\cite{hendrycks2021measuringmathematicalproblemsolving}, Minerva~\cite{lewkowycz2022solvingquantitativereasoningproblems}, Olympiad~\cite{he2024olympiadbenchchallengingbenchmarkpromoting}, GSM8K~\cite{cobbe2021trainingverifierssolvemath}, Omni~\cite{gao2024omnimathuniversalolympiadlevel}, and College~\cite{tang2024mathscalescalinginstructiontuning}; and scientific reasoning on JEEBench~\cite{arora2023llmsadvancedenoughchallenging}, TheoremQA~\cite{chen-etal-2023-theoremqa}, ScienceQA~\cite{lu2022learnexplainmultimodalreasoning}, and SciBench~\cite{wang2024scibenchevaluatingcollegelevelscientific}. Assess both mathematical problem and scientific knowledge.

\subsection{Main Results}
\noindent \textbf {Overall reasoning performance.} Figure~\ref{fig:training_dynamics} shows that SIF achieves higher training set accuracy than Intuitor throughout training and generally outperforms Vigor.
Table~\ref{tab:main_results} further evaluates SIF across 12 reasoning benchmarks.
At final checkpoint, SIF achieves average accuracies of 26.17\%, 39.19\%, and 45.29\% on Llama-3.2-3B-Instruct, Qwen2.5-Math-1.5B, and Qwen3-1.7B, respectively, yielding gains of 1.49, 15.17, and 6.11 percentage points over the initial models.
SIF outperforms Intuitor and Vigor on all three backbones and surpasses TTRL and SCOPE on both Qwen backbones.
These gains extend beyond mathematics: on Qwen2.5-Math-1.5B, SIF improves accuracy over initial model by 9.62\%, 26.21\%, and 7.69\% points on JEEBench, TheoremQA, and SciBench.

\noindent \textbf {Test-time training performance.}
Figure~\ref{fig:rollout_comparison} examines SIF with 4--12 rollouts per prompt over 80 adaptation steps on four mathematical benchmarks. All configurations improve over their initial performance. Across the entire dataset, larger rollout groups generally yield stronger performance, although the relationship is not strictly monotonic across datasets and checkpoints. Appendix~\ref{app:test_time_config}, Table~\ref{tab:test_time_improvement}, further reports checkpoint-level results on seven benchmarks. At step 80, both Pass@1 and Pass@8 exceed their initial values on every benchmark for Qwen2.5-Math-1.5B and DeepSeek-R1-Distill-Qwen-1.5B, showing that the gains persist to the predefined endpoint without correctness-based checkpoint selection.

\subsection{Ablation Study}
\label{sec:ablation}

\begin{table}[t]
    \centering
    \caption{Ablation results (\%) on five benchmarks using Qwen3-1.7B test time trained for 40 steps. Mid and Last denote the intermediate and final training
checkpoints.}
    \label{tab:ablation_benchmark_results}
    \setlength{\tabcolsep}{3pt}
    \renewcommand{\arraystretch}{1.15}
    \resizebox{\columnwidth}{!}{%
        \begin{tabular}{llcccccc}
            \toprule
            Method & Step
            & Math500
            & Minerva
            & Olympiad
            & College
            & JEEBench
             \\
            \midrule

            \rowcolor{lightblue} 
            & mid & 80.50 & 37.27 & 37.35 & 58.33 & 24.81  \\ \rowcolor{lightblue} \multirow{-2}{*}{\textbf{SIF}}
            & last & 80.30 & 36.53 & 37.24 & 57.68 & 24.47  \\

            \midrule
            \multirow{2}{*}{SIF w/o Peer}
            & mid & 79.45 & 36.17 & 34.11 & 57.13 & 23.11  \\
            & last & 79.20 & 35.85 & 33.17 & 55.86 & 23.11  \\
            \midrule
            \multirow{2}{*}{SIF w/o Window}
            & mid & 79.40 & 36.26 & 38.58 & 57.42 & 23.54  \\
            & last & 79.40 & 35.98 & 37.87 & 56.95 & 23.57 \\
           
            \midrule
            \multirow{2}{*}{Intuitor}
            & mid & 78.20 & 35.48 & 33.44 & 57.10 & 22.69  \\
            & last & 75.28 & 32.22 & 30.41 & 54.95 & 20.66  \\
             \midrule
             Vanilla & base
            & 73.02 & 33.18 & 32.91 & 54.84 & 19.22 \\
            \bottomrule
        \end{tabular}%
    }
\end{table}

\noindent \textbf {Component ablation.}
Table~\ref{tab:ablation_benchmark_results} examines the contributions of window-level and peer-level voting during test-time training.
Removing either component reduces the average accuracy across the five benchmarks: at the final checkpoint, removing peer-level voting lowers accuracy from 47.24\% to 45.44\%, while removing window-level voting lowers it to 46.75\%.
Both variants nevertheless outperform Intuitor, which achieves 42.70\%.
Full SIF achieves the highest accuracy on four of the five benchmarks at both the  intermediate and final checkpoints, although the variant without window-level voting performs better on Olympiad.
These results support the complementary contributions of the two voting components, with peer-level voting yielding a larger aggregate benefit in this setting.

\noindent \textbf {Key changes to Intuitor.}
We examine whether reward binarization or fixed-rate random
filtering can reproduce SIF's gains.
Top-2 assigns reward one to two responses with the highest
mean self-certainty per group and zero to others,
without filtering updates.
Figure~\ref{fig:Additional ablation} shows Top-2 improves
peak accuracy over Intuitor from 73.43\% to 75.89\%,
but falls to 65.54\% at step 40, versus 76.40\% for SIF.
Table~\ref{tab:ablation_validation} reports further decline
to 6.14\% at final checkpoint.
Retaining Intuitor's reward while randomly discarding
10\%, 20\%, or 40\% of responses yields similar peaks
(72.74\%--73.11\%) but final accuracies of 0\%, 3.12\%,
and 15.14\%, respectively.

\noindent \textbf{Comparison of gating strategies.}
Table~\ref{tab:gating_comparison} compares SIF with alternative gating strategies.
Random filtering discards responses uniformly at random, matching SIF's filtering rate at each training step.
Self-certainty, variance, and entropy thresholds are set away from the values observed when accuracy begins to decline in a preliminary run.
Length gating adopts SIF's 550-token threshold to isolate the effect of length-based filtering.
Random filtering underperforms Intuitor, while self-certainty, variance, and length gating improve final accuracy but lower peak accuracy.
SIF achieves the highest peak and final accuracies (73.4\% and 72.7\%), with a peak-to-final decline of only 0.7 percentage points versus 7.0 for Intuitor.
Matching filtering rates or applying length-based filtering alone is insufficient to reproduce SIF's gains.

\section{Analysis}

\begin{figure}
    \centering
    \vspace{-6pt}
\includegraphics[width=\linewidth]{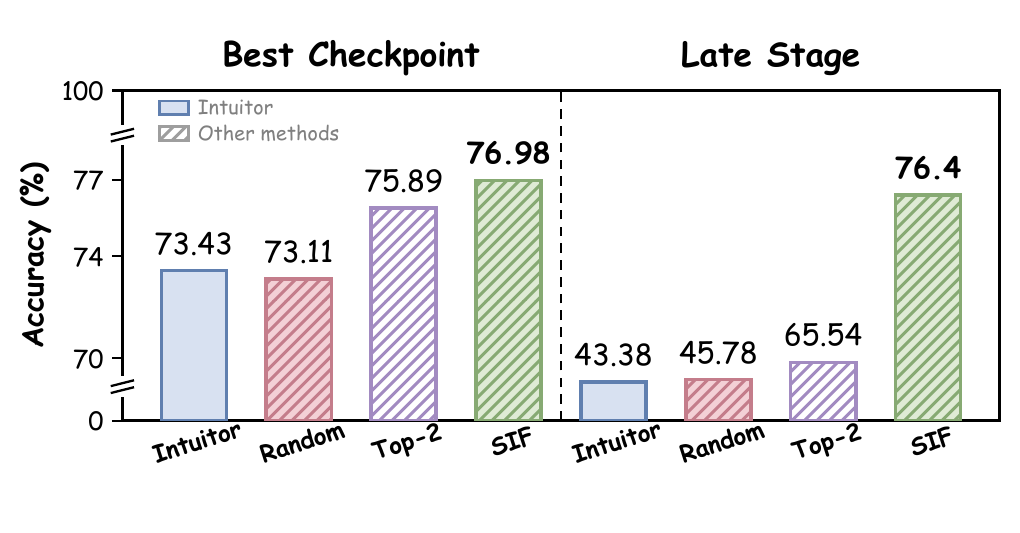}
\vspace{-32pt}
    \caption{Additional ablation on Qwen3-1.7B.}
    \vspace{-6pt}
\label{fig:Additional ablation}
\end{figure}

\subsection{Why Does Intuitor’s Deteriorate?}
\label{sec:Intuitor_collapse}
\noindent\textbf{Proxy accuracy divergence.}
Self-certainty reward in Intuitor can increase even as reasoning
accuracy declines.
Table~\ref{tab:diagnostic} in Appendix~\ref{app:ablation}
shows that Qwen3-1.7B's mean response self-certainty increases
from 55.03 to 61.54 between steps 30 and 80,
while validation accuracy decreases from 73.43\% to zero.
Self-certainty measures the divergence of token distributions
from the uniform vocabulary distribution, without directly
assessing reasoning correctness.
Moreover, averaging token-level self-certainty can yield a high
response-level score even when some tokens have low self-certainty.
Under group-relative normalization, an incorrect response
with a score above the group mean can therefore receive
a positive advantage.

\begin{table}[t]
\centering
\caption{Accuracy under different gating strategies, where Early denotes accuracy after a few training steps.}
\label{tab:gating_comparison}
\setlength{\tabcolsep}{8pt}
\renewcommand{\arraystretch}{1.15}
\small
\begin{adjustbox}{max width=\linewidth}
\begin{tabular}{lccc}
\toprule
\textbf{Method} & \textbf{Early} & \textbf{Best} & \textbf{Last} \\
\midrule
Intuitor                         & 0.453 & 0.697 & 0.627 \\
Random filtering & 0.448 & 0.676 & 0.618 \\
Self-certainty gating ($\leq 56$) & 0.457 & 0.689 & 0.643 \\
Variance gating ($\leq 58$)       & 0.451 & 0.681 & 0.656 \\
Length gating ($\geq 550$)      & 0.461 & 0.692 & 0.652 \\
Entropy gating ($\geq 0.08$)      & 0.288 & 0.667 & 0.614 \\
\rowcolor{lightblue} SIF                              & 0.563 & \textbf{0.734} & \textbf{0.727} \\
\bottomrule
\end{tabular}
\end{adjustbox}
\end{table}

\noindent\textbf{Self-certainty dispersion.}
Figure~\ref{fig:llama_qwen3_varianc} shows that Intuitor exhibits
greater within-response variance in token-level self-certainty
during late-stage training on both backbones.
The maximum gaps between the smoothed Intuitor and SIF curves
are 10.25 on Qwen3-1.7B and 6.46 on Llama-3.2-3B-Instruct.
These results complement the differing effects of ranked-band
rewards observed in Section~\ref{sec:observation}, highlighting
the limitations of summarizing token-level self-certainty
with a single response mean.
An increasing mean need not imply a uniform increase
in self-certainty across tokens.

\noindent\textbf{Lexical concentration.}
Figure~\ref{fig:top_word_share_violin} shows that, from Best to Last,
Intuitor's mean highest-frequency-word proportion increases
from 9.07\% to 11.53\%, compared with 9.35\% to 10.05\% for SIF.
The word clouds in Figure~\ref{fig:wordcloud} provide qualitative context.
The case study in Figure~\ref{fig:appendix-checkpoint-comparison}
illustrates how repetitive checking can coexist with an uncorrected error:
Intuitor's final response repeatedly enumerates and verifies divisors
but omits 142, despite explicitly computing $284/2=142$ earlier.
Its mean self-certainty nevertheless increases from 93.74
at the best checkpoint to 101.60 at the final checkpoint.
In this example, increased self-certainty and repeated checking
fail to prevent incorrect answer.

\subsection{How does SIF improve stability?}
\label{sec:sif_stability}
\noindent\textbf{Preferences across certainty bands.}
SIF constructs pairwise preferences by comparing responses
across multiple ranked self-certainty windows.
Each window contributes one vote, irrespective of the winning margin.
Under the default configuration, a response must win at least
eight of ten windows against a peer.
Consequently, large score advantages in only a few windows
cannot by themselves establish a pairwise preference.
Motivated by the differing training effects of certainty bands
observed in Section~\ref{sec:observation}, this design requires
comparative support across multiple bands rather than relying
solely on the response mean.

\begin{figure}
    \centering
\includegraphics[width=\linewidth]{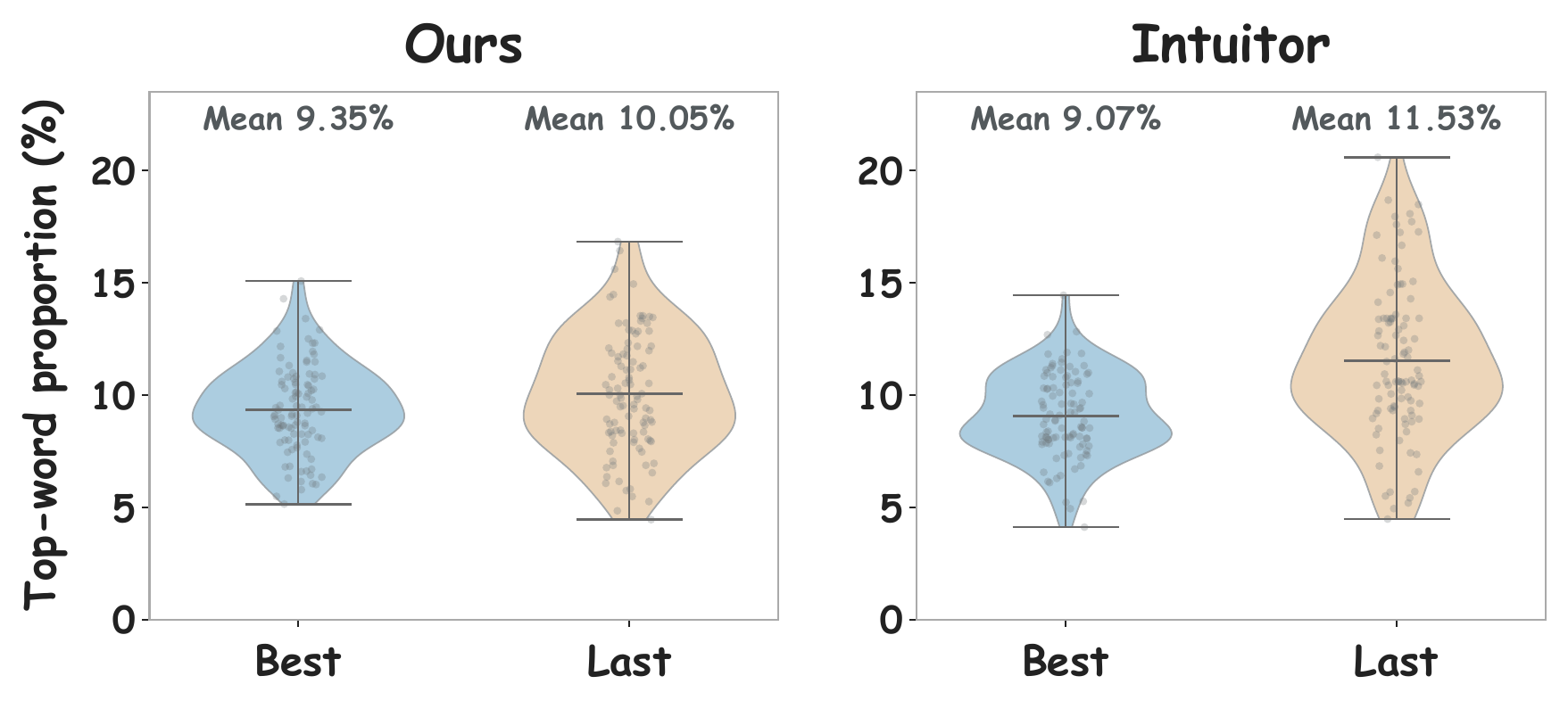}
\vspace{-19pt}
 \caption{Most frequent word proportion for SIF (left) and Intuitor (right).
Gray dots represent individual responses; horizontal lines indicate means.}
    \vspace{-2pt}
\label{fig:top_word_share_violin}
\end{figure}

\begin{table}[t]
\centering
\caption{Response accuracy (\%) before and after SIF filtering at early, mid, and final checkpoints.}
\label{tab:filtering_accuracy}
\setlength{\tabcolsep}{8pt}
\renewcommand{\arraystretch}{1.15}
\small
\begin{adjustbox}{max width=\linewidth}
\begin{tabular}{lccc}
\toprule
\textbf{Responses} & \textbf{Early} & \textbf{Mid} & \textbf{Last} \\
\midrule
Intuitor (full batch) & 62.30 & 68.55 & 24.61 \\
Vigor (full batch)    & 60.74 & 73.24 & 65.43 \\
SIF (full batch)      & 60.55 & 75.20 & 71.09 \\
\rowcolor{lightblue}
SIF (retained responses)
                     & \textbf{65.27} & \textbf{77.62} & \textbf{73.27} \\
\bottomrule
\end{tabular}
\end{adjustbox}
\end{table}

\noindent\textbf{Selective updates.}
SIF does not force a winner in every group.
Groups with identical rewards across eligible responses
are excluded from policy updates, avoiding forced preferences
based on relative rank alone.
Table~\ref{tab:filtering_accuracy} shows that retained responses
exceed full-batch accuracy by 4.72, 2.42, and 2.18 percentage
points at Early, Mid, and Last, respectively.
Thus, filtering selects a more accurate subset even late
in training, without using correctness labels.
SIF also exhibits a smaller Mid-to-Last decline in full-batch
accuracy: 4.11 percentage points, compared with 43.94 for
Intuitor and 7.81 for Vigor.
These are consistent with better performance
retention under selective internal feedback.

SIF combines multi-window and peer-level comparisons with
selective updates, requiring broader comparative support
before internal feedback is used for policy optimization.

\section{Related Work}
\noindent \textbf {Reinforcement learning for reasoning.}
RLHF aligns language models with human preferences through learned reward models~\cite{christiano2023deepreinforcementlearninghuman,ouyang2022traininglanguagemodelsfollow}.
For reasoning, RLVR uses verifiable outcomes, as demonstrated by T\"ulu~3~\cite{lambert2025tulu3pushingfrontiers}, DeepSeek-R1~\cite{Guo_2025}, and software engineering agents~\cite{da2025agentrlvrtrainingsoftwareengineering}.
GRPO, introduced in DeepSeekMath~\cite{shao2024deepseekmathpushinglimitsmathematical}, estimates advantages relative to responses to the same prompt.

\noindent \textbf {Internal feedback.}
Self-certainty supports best-of-$N$ selection~\cite{kang2025scalablebestofnselectionlarge} and serves as Intuitor's training reward, although continued optimization can degrade reasoning performance~\cite{zhang2025freelunchrethinkinginternal}.
Vigor instead uses gradient norm of response negative log-likelihood with length correction and group-relative reward shaping.
Answer agreement provides another feedback source: self-consistency aggregates sampled answers at inference~\cite{wang2023selfconsistencyimproveschainthought}, TTRL uses majority answers as pseudo-labels for test-time training~\cite{zuo2025ttrltesttimereinforcementlearning}, and SCOPE constructs subgroup-specific rewards through step wise confidence weighted voting~\cite{wang2026majorityvotingfinegrainedreliable}.

\noindent \textbf {Selection and  filtering.}
Confidence guides inference time compute allocation~\cite{Taubenfeld_2025,yang2025dynamicearlyexitreasoning}; DeepConf uses confidence statistics to filter traces and terminate unpromising generations~\cite{fu2025deepthinkconfidence}.
During training, DAPO filters groups with uniform correctness rewards and resamples informative batches~\cite{yu2025dapoopensourcellmreinforcement}, related work studies training efficiency and difficulty based filtering~\cite{zheng2025actpaysefficientreinforcement,bae2026onlinedifficultyfilteringreasoning}.
Excluding zero-advantage groups is thus an established technique.

\begin{figure}
\vspace{-24pt}
    \centering
\includegraphics[width=\linewidth]{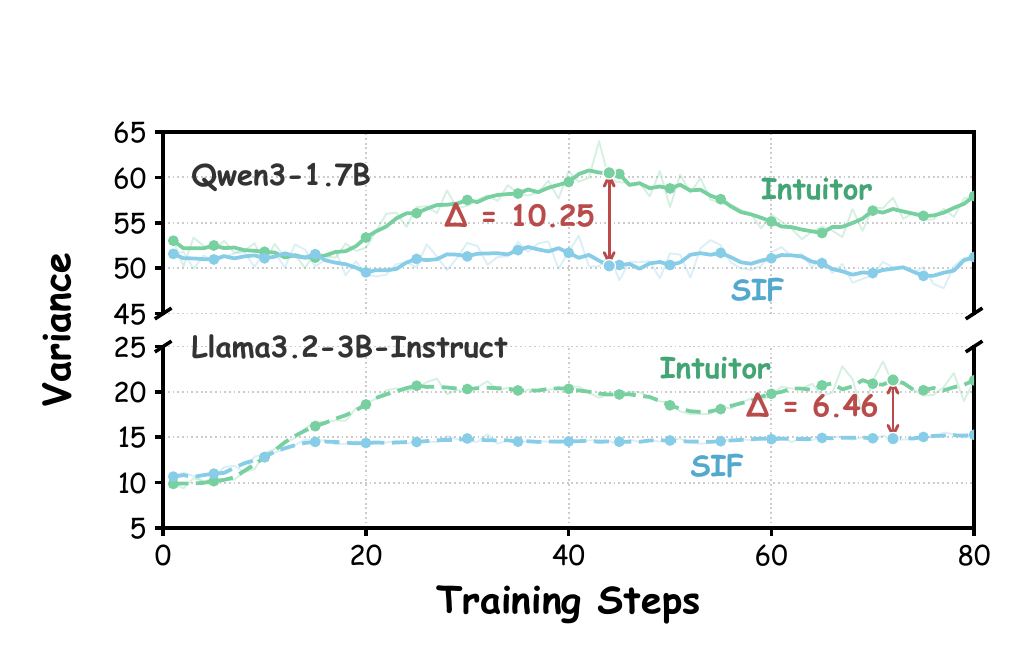}
\vspace{-19pt}
   \caption{
Variance dynamics of Intuitor and SIF on Qwen3-1.7B and Llama3.2-3B-Instruct. Arrows indicate the maximum gap between each model.
}
    \vspace{-6pt}
\label{fig:llama_qwen3_varianc}
\end{figure}

\section{Conclusion}
We introduced Selective Internal Feedback to mitigate
performance deterioration during continued self-certainty
optimization. SIF constructs binary rewards through
two-level voting over ranked token-level self-certainty
profiles and selectively updates the policy using eligible
responses and informative groups. Experiments across three
backbones and 12 reasoning benchmarks demonstrate improved
reasoning performance and better retention of training
gains, while test-time training yields sustained improvements
with as few as four rollouts per prompt. Ablations further
support the complementary contributions of voting-based
rewards and selective updates. These results highlight the
value of requiring broader comparative support for internal
feedback and controlling its use in policy optimization.

\section*{Limitations}
Due to hardware resource limitations, we were unable to conduct full-scale experiments on models with larger parameter counts. Our training dataset was relatively small; future work should involve experimentation on larger datasets. Additionally, our benchmark was primarily focused on mathematical and scientific reasoning; if conditions permit, future efforts should extend the scope to areas such as multi-hop reasoning and code generation.

\section*{Acknowledgments}
All datasets used in this work are publicly
available and contain no personal or sensitive information. Our experiments only involve open-source language models and do not
collect any human subjects data. Throughout this study, we have strictly followed
established ethical guidelines, ensuring that
our findings are reported honestly, transparently, and with full accuracy. The proposed method, \textsc{SIF}, aims to optimize the model's inference performance during a self-evolving training process that does not rely on ground-truth answers, a task that is inherently constructive and
low-risk. We do not foresee direct negative societal impacts introduced by this work. We plan to release our code and training configurations under open-source licenses to facilitate reproducibility.


% Bibliography
\nocite{*} % 强制显示 .bib 文件中的所有文献
\bibliography{custom}

\clearpage
\appendix

\section{Algorithm and Theoretical Analysis}
\label{sec:appendix_proofs}

\subsection{SIF Algorithm}
Algorithm~\ref{alg:sif} summarizes the SIF training procedure.
The peer-voting threshold $\tau_p\in(0,1]$ specifies the
minimum win fraction against other eligible responses,
with $\tau_p=3/4$ by default.

\begin{algorithm}[ht]
\caption{Selective Internal Feedback (\textsc{SIF})}
\label{alg:sif}
\small
\algrenewcommand\algorithmicindent{0.8em}
\begin{algorithmic}[1]
\Require Prompts $\mathcal D$; policies $\pi_\theta,\pi_{\mathrm{ref}}$
\Require $N,T,K,w,s,\tau_w,\tau_p,\beta,\varepsilon$;
         $\epsilon_{\mathrm{num}}>0$
\Ensure Trained policy $\pi_\theta$

\State Set $L_{\min}$ by Eq.~\eqref{eq:sif-eligibility}

\For{$\ell=1,\ldots,T$}
  \State $\pi_{\mathrm{old}}\gets\pi_\theta$;
         sample $\mathcal B\subset\mathcal D$
  \State Sample $\{y_i\}_{i=1}^{N}
         \sim\pi_{\mathrm{old}}(\cdot\mid x)$
         for each $x\in\mathcal B$
  \State Initialize token advantages $\hat A\gets0$
         and masks $M\gets0$

  \For{$x\in\mathcal B$}
    \State Obtain $\mathcal E_x$ by
           Eq.~\eqref{eq:sif-eligibility}
    \If{$|\mathcal E_x|<2$}
      \State \textbf{continue}
    \EndIf

    \State Construct $\{\mathbf z_i\}_{i\in\mathcal E_x}$
           under $\pi_{\mathrm{old}}$
    \Statex \hspace{1.6em}using
      Eqs.~\eqref{eq:sif-certainty}--\eqref{eq:sif-profile}

    \State Compute $v_{ij},b_{ij}$ using
    \Statex \hspace{1.6em}
      Eqs.~\eqref{eq:sif-window-votes}--\eqref{eq: grpo}
    \State $q_i\gets
      \sum_{j\in\mathcal E_x\setminus\{i\}}b_{ij}$
      for $i\in\mathcal E_x$
    \State $r_i\gets
      \mathbf{1}[q_i\geq\tau_p(|\mathcal E_x|-1)]$
      for $i\in\mathcal E_x$

    \If{all $\{r_i:i\in\mathcal E_x\}$ are equal}
      \State \textbf{continue}
    \EndIf

    \State $\mu_x\gets\operatorname{mean}_{i\in\mathcal E_x}r_i$;
           $\sigma_x\gets\operatorname{std}_{i\in\mathcal E_x}r_i$
    \State $\hat A_{x,i,t}\gets
           \frac{r_i-\mu_x}{\sigma_x+\epsilon_{\mathrm{num}}}$;
           $M_{x,i,t}\gets1$
    \Statex \hspace{1.6em}for $i\in\mathcal E_x$
           and $t=1,\ldots,L_i$
  \EndFor

  \If{$\sum_{x,i,t}M_{x,i,t}>0$}
    \State Optimize $\mathcal J_{\mathrm{GRPO}}(\theta)$
           with fixed $\hat A$
    \Statex \hspace{0.8em}and apply $M$ to both policy and KL terms
  \EndIf
\EndFor

\State \Return $\pi_\theta$
\end{algorithmic}
\end{algorithm}

\subsection{Theoretical Analysis}
\label{app:self_certainty_theory}
We examine the limitation of mean self-certainty identified in Sections~\ref{sec:introduction} and~\ref{sec:observation} through a constructive example. The analysis establishes that bounded, group-normalized updates can produce transient accuracy gains followed by deterioration under continued self-certainty optimization.

\noindent \textbf {Setting and scope.}
For a fixed prompt $q$, let
$\mathcal G=(o_1,\ldots,o_G)$ contain $G\geq2$ responses sampled
independently from the current policy. We write $c_\theta(q,o)$ for the
response-level self-certainty in Eq.~\eqref{eq:sif-certainty} and use
$r_i=c_\theta(q,o_i)$ as the reward. SIF's voting rewards
are defined separately in Eq.~\eqref{eq: grpo}.

Following the group-relative normalization in
Eq.~\eqref{eq: advantage}, define
\begin{equation}
\begin{aligned}
\bar r &= \frac{1}{G}\sum_{i=1}^{G}r_i,\\
\sigma_r^2 &= \frac{1}{G-1}\sum_{i=1}^{G}(r_i-\bar r)^2,\\
\hat A_{i,t}=\hat A_i
&=\frac{r_i-\bar r}{\sigma_r+\epsilon_{\mathrm{num}}}.
\end{aligned}
\label{eq:app_normalization}
\end{equation}
Here, $\sigma_r$ is the sample standard deviation, and $\epsilon_{\mathrm{num}}>0$ is a numerical stabilizer. The response-level advantage is shared across its token positions. Rewards and normalization statistics are held fixed when differentiating the policy surrogate.

We analyze continuous-time expected gradient ascent with $\beta=0$.
For $\varepsilon>0$, the clipping operation is locally inactive at
$\pi_\theta=\pi_{\mathrm{old}}$, where every importance ratio equals one.
The instantaneous gradient of the length-averaged GRPO surrogate is
therefore
\begin{equation}
g(\theta;\mathcal G)
=\frac{1}{G}\sum_{i=1}^{G}
\frac{\hat A_i}{|o_i|}
\nabla_\theta\log\pi_\theta(o_i\mid q).
\label{eq:app_group_gradient}
\end{equation}
The behavior policy is refreshed continuously, giving the dynamics
\begin{equation}
\begin{aligned}
\frac{d\theta(u)}{du}&=h(\theta(u)),\\
h(\theta)&=
\mathbb E_{\mathcal G\sim\pi_\theta^{\otimes G}}
[g(\theta;\mathcal G)],
\end{aligned}
\label{eq:app_mean_flow}
\end{equation}
where $u$ denotes optimization time and the sampling distribution is
conditioned on the fixed prompt $q$.

\begin{lemma}[Bounded group-relative advantages]
\label{lem:bounded_advantages}
For any finite rewards $r_1,\ldots,r_G$ with $G\geq2$, the advantages
in Eq.~\eqref{eq:app_normalization} satisfy
\begin{equation}
\begin{aligned}
|\hat A_i|&\leq\frac{G-1}{\sqrt G}\leq\sqrt{G-1},\\
\frac{1}{G}\sum_{i=1}^{G}\hat A_i^2
&\leq\frac{G-1}{G}\leq1.
\end{aligned}
\label{eq:advantage_bound}
\end{equation}
\end{lemma}

\begin{proof}
Let $v_i=r_i-\bar r$, so that $\sum_i v_i=0$ and
$\sum_i v_i^2=(G-1)\sigma_r^2$. By the Cauchy--Schwarz inequality,
\begin{equation}
\begin{aligned}
v_i^2
&=\left(\sum_{j\neq i}v_j\right)^2
\leq(G-1)\sum_{j\neq i}v_j^2\\
&=(G-1)\big((G-1)\sigma_r^2-v_i^2\big).
\end{aligned}
\end{equation}
Hence $|v_i|\leq(G-1)\sigma_r/\sqrt G$. Since
$\epsilon_{\mathrm{num}}>0$, we obtain
\begin{equation}
\begin{aligned}
|\hat A_i|
&\leq\frac{G-1}{\sqrt G}
\frac{\sigma_r}{\sigma_r+\epsilon_{\mathrm{num}}}
\leq\frac{G-1}{\sqrt G},\\
\frac{1}{G}\sum_{i=1}^{G}\hat A_i^2
&=\frac{G-1}{G}
\frac{\sigma_r^2}{(\sigma_r+\epsilon_{\mathrm{num}})^2}
\leq\frac{G-1}{G}.
\end{aligned}
\end{equation}
When $\sigma_r=0$, all rewards are equal and every advantage is zero,
so the same bounds apply.
\end{proof}

The lemma bounds normalized advantages independently of reward magnitude. The following construction additionally bounds the policy gradient itself while exhibiting a divergence between self-certainty and correctness.

\begin{theorem}[Transient accuracy gains and collapse]
\label{thm:exact_collapse}
Fix $G\geq2$ and $\epsilon_{\mathrm{num}}>0$. There exists a two-token
autoregressive policy with a scalar parameter $\theta$, initialized at
$\theta(0)=0$, for which the dynamics in
Eq.~\eqref{eq:app_mean_flow} have the following properties:
\begin{enumerate}
    \item Greedy accuracy on a single-prompt task follows a
    $0\to1\to0$ transition.
    \item The probability of sampling a correct response first increases
    to a unique maximum and then strictly decreases to zero.
    \item After this maximum, expected mean self-certainty strictly
    increases and diverges to infinity.
\end{enumerate}
Throughout the process, every sampled group satisfies
$0\leq g(\theta;\mathcal G)\leq1/2$, and all second-token conditional
distributions remain fixed.
\end{theorem}

\begin{proof}
\textit{Policy and reward construction.}
Take $\mathcal V=\{B,C,W\}$ and let $U$ be uniform on this vocabulary.
For the fixed prompt $q$, a response $o=(o_1,o_2)$ is correct exactly
when $o_1=C$. Thus, the first token selects a correct or incorrect
branch, and correctness is independent of the second token.

Assign branch constants $(d_B,d_C,d_W)=(0,1,2)$ and choose fixed,
full-support distributions $\nu_b$ satisfying
\begin{equation}
\mathrm{KL}(U\|\nu_b)=2d_b,
\qquad b\in\mathcal V.
\end{equation}
These distributions exist because
$\nu(a)=(a,(1-a)/2,(1-a)/2)$ has full support for $a\in[1/3,1)$.
Along this path, $\mathrm{KL}(U\|\nu(a))$ is continuous, equals zero
at $a=1/3$, and tends to infinity as $a\to1$.

Let $p_0(B)=0.80$, $p_0(C)=0.19$, and $p_0(W)=0.01$, and define
\begin{equation}
\begin{aligned}
p_\theta(b)&=\frac{p_0(b)e^{d_b\theta}}{Z(\theta)},\\
Z(\theta)&=\sum_{b\in\mathcal V}p_0(b)e^{d_b\theta}.
\end{aligned}
\label{eq:app_routing_policy}
\end{equation}
The complete response policy is
$\pi_\theta(o\mid q)=p_\theta(o_1)\nu_{o_1}(o_2)$.
By Eq.~\eqref{eq:sif-certainty}, its response-level reward:
\begin{equation}
c_\theta(q,o)
=\frac12\mathrm{KL}(U\|p_\theta)+d_{o_1}.
\label{eq:app_constructed_reward}
\end{equation}
This reward uses only the policy's token distributions; the correctness
indicator is used solely to evaluate the construction.

\medskip\noindent
\textit{Bounded updates and the parameter trajectory.}
Write $b_i=o_{i,1}$ and
$\bar d=G^{-1}\sum_{i=1}^{G}d_{b_i}$ for a sampled group.
The shared first-token contribution cancels under reward centering:
\begin{equation}
\begin{aligned}
r_i-\bar r&=d_{b_i}-\bar d,\\
\sigma_r^2&=\frac{1}{G-1}
\sum_{i=1}^{G}(d_{b_i}-\bar d)^2.
\end{aligned}
\end{equation}
Let $\mu(\theta)=\mathbb E_{b\sim p_\theta}[d_b]$.
Since the second-token distributions are fixed,
\begin{equation}
\nabla_\theta\log\pi_\theta(o_i\mid q)
=d_{b_i}-\mu(\theta).
\end{equation}
Substituting into Eq.~\eqref{eq:app_group_gradient} gives
\begin{equation}
\begin{aligned}
g(\theta;\mathcal G)
&=\frac{1}{2G}\sum_{i=1}^{G}
\frac{(d_{b_i}-\bar d)(d_{b_i}-\mu(\theta))}
{\sigma_r+\epsilon_{\mathrm{num}}}\\
&=\frac{(G-1)\sigma_r^2}
{2G(\sigma_r+\epsilon_{\mathrm{num}})}.
\end{aligned}
\label{eq:app_scalar_gradient}
\end{equation}
The second equality uses $\sum_i(d_{b_i}-\bar d)=0$.
Because $d_b\in[0,2]$,
\begin{equation}
\begin{aligned}
\frac{G-1}{G}\sigma_r^2
&=\frac1G\sum_{i=1}^{G}(d_{b_i}-1)^2
-(\bar d-1)^2\\
&\leq1.
\end{aligned}
\end{equation}
Consequently,
\begin{equation}
0\leq g(\theta;\mathcal G)
\leq\frac{G-1}{2G}\sigma_r
\leq\frac12\sqrt{\frac{G-1}{G}}
\leq\frac12.
\end{equation}

Every branch has positive probability at each finite $\theta$.
Since $G\geq2$, groups containing different branches occur with positive
probability and have $\sigma_r>0$. Thus $h(\theta)>0$ at every finite
$\theta$. For each fixed branch configuration,
Eq.~\eqref{eq:app_scalar_gradient} gives a constant group gradient.
Thus $h$ is a finite weighted sum of constants with smooth sampling
probabilities. It is smooth and bounded, so the scalar flow has a
unique solution for all $u\geq0$, and $\theta(u)$ strictly increases.

If $\theta(u)$ converged to a finite value $L$, continuity would give
$h(\theta(u))\to h(L)>0$, contradicting convergence. Hence
\begin{equation}
\lim_{u\to\infty}\theta(u)=+\infty.
\label{eq:app_parameter_limit}
\end{equation}

\medskip\noindent
\textit{Greedy and sampled correctness.}
The comparisons $p_\theta(B)=p_\theta(C)$ and
$p_\theta(C)=p_\theta(W)$ give
\begin{equation}
\begin{aligned}
\theta_1&=\log(0.80/0.19)\approx1.438,\\
\theta_2&=\log(0.19/0.01)\approx2.944.
\end{aligned}
\end{equation}
Break first-token ties according to the fixed priority $B,C,W$.
The greedy first token is then
\begin{equation}
\operatorname{Greedy}_1(\theta)=
\begin{cases}
B, & 0\leq\theta\leq\theta_1,\\
C, & \theta_1<\theta\leq\theta_2,\\
W, & \theta>\theta_2.
\end{cases}
\end{equation}
Because $\theta(u)$ increases from zero to infinity, greedy accuracy
follows the claimed $0\to1\to0$ transition.

The probability of sampling a correct response is $p_\theta(C)$, with
\begin{equation}
\begin{aligned}
\frac{d}{d\theta}p_\theta(C)
&=p_\theta(C)\big(1-\mu(\theta)\big),\\
\mu'(\theta)&=\operatorname{Var}_{p_\theta}(d)>0.
\end{aligned}
\end{equation}
Here, $\mu(0)=0.21$, and $\mu(\theta)=1$ exactly when
$p_\theta(B)=p_\theta(W)$. This occurs at
\begin{equation}
\theta^*=\frac12\log(0.80/0.01)\approx2.191.
\end{equation}
Therefore, $p_\theta(C)$ strictly increases up to $\theta^*$ and
strictly decreases thereafter, attaining
\begin{equation}
p_{\theta^*}(C)
=\frac{0.19}{0.19+2\sqrt{0.80\cdot0.01}}
\approx0.5151.
\end{equation}
Finally,
\begin{equation}
p_\theta(C)
=\frac{0.19e^\theta}{0.80+0.19e^\theta+0.01e^{2\theta}}
\longrightarrow0
\end{equation}
as $\theta\to+\infty$, and therefore as $u\to\infty$.

\medskip\noindent
\textit{Increasing self-certainty during deterioration.}
Define the expected response-level self-certainty by
\begin{equation}
\begin{aligned}
R(\theta)
&=\mathbb E_{o\sim\pi_\theta(\cdot\mid q)}[c_\theta(q,o)]\\
&=\frac12\mathrm{KL}(U\|p_\theta)+\mu(\theta).
\end{aligned}
\end{equation}
Since $U$ is uniform and $(d_B+d_C+d_W)/3=1$,
\begin{equation}
\mathrm{KL}(U\|p_\theta)
=\mathrm{KL}(U\|p_0)+\log Z(\theta)-\theta.
\end{equation}
Using $(\log Z)'=\mu$ yields
\begin{equation}
R'(\theta)
=\frac12\big(\mu(\theta)-1\big)
+\operatorname{Var}_{p_\theta}(d)>0
\end{equation}
for $\theta\geq\theta^*$. Thus expected self-certainty strictly
increases throughout the decline in correct sampling probability.
Furthermore, $Z(\theta)\geq0.01e^{2\theta}$ implies that
$\mathrm{KL}(U\|p_\theta)\to\infty$, so $R(\theta)\to\infty$.
All conditional distributions $\nu_b$ remain fixed by construction,
which completes the proof.
\end{proof}

\noindent \textbf {Implications for internal feedback.}
It must be acknowledged that, given the high complexity inherent in the actual sampling and generation processes of large language models, we have significantly simplified this process to derive an illustrative, simplified model. This example formalizes a possible divergence between mean
self-certainty and correctness under continued optimization. The fixed conditional distributions isolate a
change in branch selection as the source of deterioration in this
constructed policy.

The result supports the motivation for examining reward granularity and update participation in Section~\ref{sec:observation}.
SIF addresses these design choices through ranked-profile comparisons, two-level voting, and selective updates. The benefits of these components are evaluated empirically. The theoretical result concerns the specified continuous-time dynamics with $\beta=0$. Analyzing finite-step updates with KL regularization, or establishing stability guarantees for SIF, requires separate analysis.

\section{Experimental Setup and Reproducibility}
The main experiments in our paper fall into two categories: one involves training at test time, while the other entails RLIF training on the MATH dataset followed by evaluation on benchmarks. 

\begin{table*}[t]
\centering
\caption{
Pass@1 and Pass@8 performance across training steps.
$\Delta$ denotes the improvement from Base to Step 80
in percentage points.
}
\label{tab:test_time_improvement}
\setlength{\tabcolsep}{3.2pt}
\renewcommand{\arraystretch}{1.15}
\small
\resizebox{\textwidth}{!}{
\begin{tabular}{l*{14}{c}}
\toprule
\multirow{2.5}{*}{\textbf{Checkpiont}}
& \multicolumn{2}{c}{MATH}
& \multicolumn{2}{c}{MATH500}
& \multicolumn{2}{c}{Minerva}
& \multicolumn{2}{c}{Olympiad}
& \multicolumn{2}{c}{CollegeMath}
& \multicolumn{2}{c}{JEEBench}
& \multicolumn{2}{c}{TheoremQA} \\
\cmidrule(lr){2-3}
\cmidrule(lr){4-5}
\cmidrule(lr){6-7}
\cmidrule(lr){8-9}
\cmidrule(lr){10-11}
\cmidrule(lr){12-13}
\cmidrule(lr){14-15}
& Pass@1 & Pass@8
& Pass@1 & Pass@8
& Pass@1 & Pass@8
& Pass@1 & Pass@8
& Pass@1 & Pass@8
& Pass@1 & Pass@8
& Pass@1 & Pass@8 \\
\midrule

\multicolumn{15}{l}{\textbf{Backbone Model: Qwen2.5-Math-1.5B}} \\

Vanilla 
& 55.35 & 85.88 & 50.00 & 85.20
& 14.84 & 39.34 & 25.43 & 52.08
& 31.96 & 63.17 & 12.11 & 46.60
& 13.91 & 37.88 \\

Global Step 10
& 57.09 & 86.12 & 52.20 & 86.60
& 15.95 & 40.44 & 26.63 & 52.37
& 34.92 & 64.30 & 12.74 & 47.96
& 15.62 & 40.75 \\

Global Step 20
& 60.50 & 86.36 & 58.70 & 88.20
& 17.92 & 43.38 & 28.23 & 53.41
& 40.60 & 66.54 & 16.12 & 51.65
& 20.97 & 44.75 \\

Global Step 30
& 64.22 & 87.36 & 62.62 & 89.60
& 20.36 & 42.65 & 29.45 & 55.04
& 47.40 & 67.00 & 17.26 & 51.65
& 24.39 & 46.75 \\

Global Step 40
& 67.56 & 88.74 & 65.10 & 88.20
& 21.55 & 43.75 & 31.82 & 54.60
& 51.96 & 67.21 & 16.89 & 50.29
& 25.50 & 45.50 \\

Global Step 80
& 71.45 & 88.90 & 69.25 & 88.60
& 23.81 & 44.49 & 33.72 & 56.82
& 54.97 & 67.10 & 15.51 & 50.87
& 27.17 & 47.00 \\

\rowcolor{lightblue}
$\Delta$
& +16.10 & +3.02 & +19.25 & +3.40
& +8.97 & +5.15 & +8.29 & +4.74
& +23.01 & +3.93 & +3.40 & +4.27
& +13.26 & +9.12 \\
\midrule

\multicolumn{15}{l}{\textbf{Backbone Model: DeepSeek-R1-Distill-Qwen-1.5B}} \\

Vanilla 
& 69.99 & 90.06 & 67.27 & 88.60
& 28.03 & 48.53 & 32.72 & 55.93
& 51.73 & 67.32 & 16.99 & 45.24
& 25.50 & 45.75 \\

Global Step 10
& 72.47 & 90.58 & 69.85 & 88.40
& 29.69 & 52.21 & 34.27 & 57.27
& 52.35 & 68.24 & 18.93 & 48.35
& 25.86 & 47.38 \\

Global Step 20
& 75.85 & 92.54 & 73.15 & 90.40
& 28.77 & 50.00 & 37.50 & 59.35
& 53.55 & 68.10 & 20.68 & 50.49
& 26.66 & 46.50 \\

Global Step 30
& 77.95 & 93.48 & 74.52 & 89.80
& 27.94 & 47.06 & 39.37 & 59.50
& 55.84 & 69.16 & 21.24 & 52.04
& 27.77 & 47.12 \\

Global Step 40
& 78.14 & 92.86 & 76.80 & 92.00
& 29.78 & 48.90 & 40.39 & 60.83
& 56.75 & 69.62 & 21.38 & 51.65
& 28.03 & 47.38 \\

Global Step 80
& 79.08 & 93.30 & 76.62 & 91.20
& 30.84 & 49.26 & 40.93 & 63.50
& 57.40 & 69.87 & 22.65 & 53.20
& 28.59 & 48.38 \\

\rowcolor{lightblue}
$\Delta$
& +9.09 & +3.24 & +9.35 & +2.60
& +2.81 & +0.73 & +8.21 & +7.57
& +5.67 & +2.55 & +5.66 & +7.96
& +3.09 & +2.63 \\
\bottomrule
\end{tabular}
}
\end{table*}

\subsection{Test-Time Training}
\label{app:test_time_config}
We performed dataset-level test-time training using the internal voting rewards of \textsc{SIF}. Each run was initialized from the selected pretrained model. Training used the test split of each mathematical or scientific benchmark. The training budget was fixed at 80 steps. Each step sampled 64 prompts and generated eight responses per prompt, yielding 512 candidate responses before masking. Table~\ref{tab:test_time_improvement} presents the performance of SIF using test-time training on the Qwen2.5-Math-1.5B and DeepSeek-R1-Distill-Qwen-1.5B models across various benchmarks. Table~\ref{tab:tta_optimization} summarizes the optimization settings, and Table~\ref{tab:tta_voting_generation} reports the voting and generation settings.

\noindent \textbf {Test-time performance across training checkpoints.}
Table~\ref{tab:test_time_improvement} further evaluates SIF checkpoints using Pass@1 and Pass@8. From the initial model to training step 80, both metrics improve on all seven evaluated benchmarks for Qwen2.5-Math-1.5B and DeepSeek-R1-Distill-Qwen-1.5B. Averaged across these benchmarks, Qwen2.5-Math gains 13.18 points in Pass@1 and 4.80 points in Pass@8, while the distilled DeepSeek model gains 6.27 and 3.90 points, respectively. For Qwen2.5-Math, the largest Pass@1 gain occurs on CollegeMath (+23.01 points), whereas the largest Pass@8 gain occurs on TheoremQA (+9.12 points). For the distilled model, Pass@8 improves by 7.57 points on Olympiad and 7.96 points on JEEBench. Although individual trajectories fluctuate, the positive endpoint gains under both evaluation measures show that SIF training improves test-time performance beyond the initial checkpoints.

\begin{table}[t]
    \centering
    \small
    \setlength{\tabcolsep}{4pt}
     \caption{Optimization settings for test-time training.
    The PPO mini-batch size is expressed in prompts; each
    prompt contributes eight sampled responses before masking.
    Run-specific exceptions are described in the text.}
    \renewcommand{\arraystretch}{1.08}
    \begin{tabularx}{\columnwidth}{
        @{}>{\raggedright\arraybackslash}X
        >{\raggedright\arraybackslash}p{0.43\columnwidth}@{}
    }
        \toprule
        \textbf{Parameter} & \textbf{Value} \\
        \midrule
        Training budget & 80 steps \\
        Prompt batch size & 64 \\
        Responses per prompt & 8 \\
        PPO mini-batch size & 64 \\
        Micro-batch per GPU & 4 \\
        PPO epochs per batch & 1 \\
        Optimizer & AdamW \\
        Adam coefficients & $(0.9, 0.999)$ \\
        Peak learning rate & $3 \times 10^{-6}$ \\
        Learning-rate schedule & Cosine decay \\
        Warmup & 8 steps (10\%) \\
        Weight decay & 0.01 \\
        PPO clipping range & $[0.8, 1.2]$ \\
        Gradient clipping norm & 1.0 \\
        KL loss coefficient & 0.005 \\
        KL estimator & \texttt{low\_var\_kl} \\
        KL penalty in reward & Disabled \\
        Entropy coefficient & 0 \\
        Parameter updates & Full-model fine-tuning \\
        Gradient checkpointing & Enabled \\
        Distributed strategy & FSDP \\
        Hardware allocation & 1 node, 8 A100 40G GPUs \\
        Random seed & 123 \\
        Checkpoint steps & 10, 20, 30, 40, 80 \\
        \bottomrule
    \end{tabularx}
    \label{tab:tta_optimization}
\end{table}

\noindent \textbf {Internal voting rewards.}
For each response, valid tokens were sorted in descending
order of self-certainty.
We computed mean self-certainty within ten windows of
100 tokens each.
The windows spanned the sorted sequence from the highest-
to the lowest-certainty tokens, with starting positions
distributed evenly between the two endpoints.
A response won a pairwise comparison when its window mean
was strictly higher in at least eight of the ten windows.
A response received a binary reward of one when its
win fraction against the other eligible responses
to the same prompt reached the peer-voting threshold
$\tau_p=3/4$, and zero otherwise.
These rewards were normalized within each prompt group
using the GRPO mean and standard deviation.

The minimum eligible response length was
$100 + 50 \times (10-1) = 550$ tokens.
At this length, consecutive windows had a stride of 50 tokens;
for longer responses, the spacing increased with sequence length.
Responses shorter than 550 tokens were excluded from updates.
Prompt groups with identical rewards across all eligible
responses were also excluded.
The resulting mask was applied to both the policy-gradient
and KL losses.

\begin{table}[t]
    \centering
    \small
    \setlength{\tabcolsep}{4pt}
    \caption{Voting and generation settings. $^\dagger$The DeepSeek-R1-Distill-Qwen-1.5B uses a maximum prompt length of 4,096 tokens.}
    \renewcommand{\arraystretch}{1.08}
    \begin{tabularx}{\columnwidth}{
        @{}>{\raggedright\arraybackslash}X
        >{\raggedright\arraybackslash}p{0.43\columnwidth}@{}
    }
        \toprule
        \textbf{Parameter} & \textbf{Value} \\
        \midrule
        \multicolumn{2}{@{}l}{\textit{Internal voting}} \\
        Number of windows & 10 \\
        Tokens per window & 100 \\
        Minimum-length stride & 50 tokens \\
        Window wins required & 8 of 10 \\
        Peer-voting threshold $\tau_p$ & $3/4$ \\
        Reward values & $\{0,1\}$ \\
        Advantage normalization & Group mean and std. \\
        \midrule
        \multicolumn{2}{@{}l}{\textit{Response generation}} \\
        Rollout engine & vLLM \\
        Rollout precision & bfloat16 \\
        Tensor parallel size & 1 \\
        GPU memory utilization & 0.6 \\
        Maximum prompt length & 1,024 tokens$^\dagger$ \\
        Maximum response length & 3,072 tokens \\
        Training temperature & 1.0 \\
        Training top-$p$ & 1.0 \\
        Training top-$k$ & Disabled \\
        Evaluation decoding & Sampling, eight responses per prompt \\
Accuracy aggregation & Mean accuracy over eight responses \\
        Evaluation steps & 10, 20, 30, 40, 80 \\
        \bottomrule
    \end{tabularx}
    \label{tab:tta_voting_generation}
\end{table}

\subsection{MATH Training}
\label{app:math_training}
This appendix documents all models training run used
in the method comparison.
We trained \textsc{SIF} on the MATH training split for
one epoch, with checkpoint resumption disabled.
Prompts exceeding 512 tokens were filtered, leaving
7,420 of the original 7,500 training examples.
The MATH test split served as the online validation set;
4,965 of its 5,000 examples remained after filtering.
With a prompt batch size of 64 and the final incomplete
batch dropped, the run completed 115 training steps.

Table~\ref{tab:math_optimization} summarizes the
optimization settings.
Each training step sampled eight responses for each of
64 prompts, producing 512 candidate responses before
the update mask was applied.

\begin{table}[t]
    \centering
    \small
    \setlength{\tabcolsep}{4pt}
    \caption{Optimization settings for all models
    MATH training run.
    The PPO mini-batch size is expressed in prompts;
    each prompt contributes eight sampled responses.}
    \renewcommand{\arraystretch}{1.08}
    \begin{tabularx}{\columnwidth}{
        @{}>{\raggedright\arraybackslash}X
        >{\raggedright\arraybackslash}p{0.43\columnwidth}@{}
    }
        \toprule
        \textbf{Parameter} & \textbf{Value} \\
        \midrule
        Initial model & - \\
        Training split & MATH train \\
        Online validation split & MATH test \\
        Training epochs & 1 \\
        Training steps & 115 \\
        Prompt batch size & 64 \\
        Responses per prompt & 8 \\
        PPO mini-batch size & 64 \\
        Micro-batch per GPU & 4 \\
        PPO epochs per batch & 1 \\
        Optimizer & AdamW \\
        Adam coefficients & $(0.9, 0.999)$ \\
        Peak learning rate & $3 \times 10^{-6}$ \\
        Learning-rate schedule & Cosine decay \\
        Warmup & 11 steps (ratio 0.1) \\
        Weight decay & 0.01 \\
        PPO clipping range & $[0.8, 1.2]$ \\
        Gradient clipping norm & 1.0 \\
        KL loss coefficient & 0.005 \\
        KL estimator & \texttt{low\_var\_kl} \\
        KL penalty in reward & Disabled \\
        Entropy coefficient & 0 \\
        Parameter updates & Full-model fine-tuning \\
        Distributed strategy & FSDP \\
        Gradient checkpointing & Enabled \\
        Hardware allocation & 1 node, 8 A100 40G GPUs \\
        Random seed & 123 \\
        \bottomrule
    \end{tabularx}
    \label{tab:math_optimization}
\end{table}

\noindent \textbf {Internal voting rewards.} For each response, valid tokens were sorted in descending
order of self-certainty.
Ten windows of 100 tokens were placed across the sorted
sequence, with starting positions distributed evenly
between the first and last possible windows.
Each window was represented by its mean self-certainty.

A response won a pairwise comparison when its window mean
was strictly higher in at least eight of the ten windows.
A response received a binary reward of one when its
win fraction against the other eligible responses
to the same prompt reached the peer-voting threshold
$\tau_p=3/4$, and zero otherwise.
Advantages were computed by normalizing these rewards
within each prompt group using the group mean and
standard deviation.

The minimum eligible response length was
$100 + 50 \times (10-1) = 550$ tokens.
At this length, adjacent windows had a stride of 50 tokens;
for longer responses, their spacing increased with length.
Responses shorter than 550 tokens and prompt groups with
identical rewards across eligible responses were excluded
from updates.
The update mask was applied to both the policy-gradient
and KL losses.

\begin{table}[t]
    \centering
    \small
    \setlength{\tabcolsep}{4pt}
     \caption{Internal voting and training generation
    settings for all models MATH run.}
    \renewcommand{\arraystretch}{1.08}
    \begin{tabularx}{\columnwidth}{
        @{}>{\raggedright\arraybackslash}X
        >{\raggedright\arraybackslash}p{0.43\columnwidth}@{}
    }
        \toprule
        \textbf{Parameter} & \textbf{Value} \\
        \midrule
        \multicolumn{2}{@{}l}{\textit{Internal voting}} \\
        Number of windows & 10 \\
        Tokens per window & 100 \\
        Minimum-length stride & 50 tokens \\
        Window wins required & 8 of 10 \\
        Peer-voting threshold $\tau_p$ & $3/4$ \\
        Reward values & $\{0,1\}$ \\
        \midrule
        \multicolumn{2}{@{}l}{\textit{Training generation}} \\
        Rollout engine & vLLM \\
        Rollout precision & bfloat16 \\
        Tensor parallel size & 1 \\
        GPU memory utilization & 0.6 \\
        Maximum prompt length & 512 tokens \\
        Maximum response length & 3,072 tokens \\
        Temperature & 1.0 \\
        Top-$p$ & 1.0 \\
        Top-$k$ & Disabled \\
        \bottomrule
    \end{tabularx}
    \label{tab:math_voting_generation}
\end{table}

\noindent \textbf {Online validation and checkpointing.}
Online validation used the MATH test split, with one
greedily decoded response per prompt and a temperature
of zero.
Validation and checkpoint checks were scheduled every
ten training steps.
The checkpoint policy retained the best and last
checkpoints.
Validation before training was disabled.

\noindent \textbf {Independent method-comparison evaluation.}
The downstream comparison evaluator sampled eight
responses per prompt with a temperature of 0.6,
top-$p$ of 1.0, and random seed 123.
It used maximum prompt and response lengths of
1,024 and 3,072 tokens, respectively.
Evaluation ran with vLLM on two GPUs, using a tensor
parallel size of two and a GPU memory utilization of 0.85.

\subsection{Ablation study configurations.}
Table~\ref{tab:ablation_benchmark_results} compares full SIF
with variants that remove peer-level or window-level voting.
Full SIF constructs binary rewards through both voting levels.
\textit{SIF w/o Peer} retains window-level pairwise comparisons
and length-based eligibility, but uses the fraction of eligible
peers defeated as a continuous reward without applying the
peer-voting threshold.
\textit{SIF w/o Window} removes both window-level comparisons
and the minimum response-length requirement.
All sampled responses participate in pairwise comparisons
based on their mean token-level self-certainty.
A response receives a binary reward of one when its mean
self-certainty exceeds that of at least three quarters of
the other sampled responses to the same prompt, and zero
otherwise.
Both variants retain group-level filtering and GRPO-style
advantage normalization over their respective comparison groups.
Intuitor uses mean token-level self-certainty over each
complete response as its continuous reward.
All configurations use Qwen3-1.7B and undergo 40 steps
of test-time training.

\begin{table}[t]
    \centering
    \caption{Diagnostics for \textsc{Intuitor}. Acc. is MATH validation accuracy; SC and entropy are the logged training-batch means. These snapshots illustrate proxy--accuracy divergence.}
    \small
    \setlength{\tabcolsep}{4pt}
    \begin{tabular}{lrrrr}
        \toprule
        Backbone & Step & Acc. (\%) & SC & Entropy \\
        \midrule
        Qwen3-1.7B & 30 & 73.43 & 55.03 & 0.0957 \\
                    & 80 & 0.00 & 61.54 & 0.0280 \\
        \midrule
        Llama-3.2-3B & 20 & 41.86 & 33.29 & 0.0721 \\
        (Instruct)   & 40 & 0.00 & 35.98 & 0.0161 \\
        \bottomrule
    \end{tabular}
    \label{tab:diagnostic}
\end{table}

\section{Ablations of Random Filtering and Reward Discretization}
\label{app:ablation}
We examine whether reducing the number of training responses or
discretizing mean self-certainty rewards is sufficient to prevent
performance collapse. We evaluate three fixed-rate random-filtering
variants of Intuitor and a top-2 reward baseline on Qwen3-1.7B.
Table~\ref{tab:ablation_validation} reports their validation trajectories
and peak-to-final accuracy drops.

\subsection{Experimental Protocol}
\label{app:ablation_protocol}
All four configurations were trained for one epoch on MATH, with one
run per configuration. Each rollout step sampled 64 prompts and
generated eight responses per prompt, yielding 512 responses before
filtering. Dropping the final incomplete batch resulted in 115 rollout
steps per run.

The sampling temperature was 1.0, with maximum prompt and response
lengths of 512 and 3,072 tokens, respectively. The actor used a peak
learning rate of $3\times10^{-6}$, a cosine schedule with 11 warmup
steps, one PPO epoch per batch, a clipping parameter of 0.2, and a
KL loss coefficient of 0.005. Policy optimization used GRPO
group-relative advantages. Reference answers were used only for
monitoring and evaluation, not for reward construction.

Validation used greedy decoding with one response per prompt on the
MATH test split. Accuracy was recorded at rollout steps
$10,20,\ldots,110$ and at the final step, 115. Peak accuracy denotes
the highest value among these checkpoints, and Drop is the
peak-to-final decrease in percentage points.

\subsection{Ablation Definitions}
\label{app:ablation_definitions}
For each prompt $x$, let $\{y_i\}_{i=1}^{N}$ denote the sampled
responses, where $N=8$. Let $\bar{c}_i$ be the mean token-level
self-certainty of $y_i$, as defined in
Equation~\ref{eq:sif-certainty}.

\begin{table*}[t]
    \centering
    \caption{
        Validation accuracy of Intuitor with random response filtering
        and a top-2 reward baseline on Qwen3-1.7B trained on MATH,
        with one run per configuration.
        Random filtering excludes the indicated percentage of sampled
        responses from each training batch.
        The top-2 baseline assigns binary rewards according to
        mean self-certainty rankings without response filtering.
        Peak denotes the highest accuracy among the reported checkpoints;
        Drop is Peak minus Final in percentage points,
        calculated from the displayed values.
    }
    \label{tab:ablation_validation}
    \small
    \setlength{\tabcolsep}{5pt}
    \renewcommand{\arraystretch}{1.12}
    \begin{tabularx}{\textwidth}{
        @{}>{\raggedright\arraybackslash}p{0.30\textwidth}
        *{4}{>{\centering\arraybackslash}X}@{}
    }
        \toprule
        Metric / rollout step
        & \makecell{Filter responses\\Random 40\%}
        & \makecell{Filter responses\\Random 20\%}
        & \makecell{Filter responses\\Random 10\%}
        & \makecell{Top-2\\no filter} \\
        \midrule
        \multicolumn{5}{@{}l}{\textit{Validation accuracy (\%)}} \\
        Global Step 10  & 62.20 & 63.04 & 59.68 & 60.18 \\
        Global Step 20  & 67.65 & 72.80 & 73.11 & 75.89 \\
        Global Step 30  & 72.74 & 69.33 & 63.11 & 68.98 \\
        Global Step 40  & 52.55 & 54.87 & 45.78 & 65.54 \\
        Global Step 50  & 43.46 & 41.66 & 30.07 & 65.78 \\
        Global Step 60  & 39.19 & 33.78 & 15.67 &  64.69\\
        Global Step 70  & 32.44 & 13.76 & 11.45 & 54.48 \\
        Global Step 80  & 29.47 & 15.79 & 7.14 & 21.83 \\
        Global Step 90  & 25.49 & 8.52 &  0 & 10.21 \\
        Global Step 100 & 19.30 & 6.21 & 0 & 7.23  \\
        Global Step 110 & 17.74 & 4.82 & 0 & 6.67  \\
        Global Step 115 (Final)
            & 15.14 & 3.12 & 0 & 6.14 \\
        \midrule
        \multicolumn{5}{@{}l}{\textit{Peak and retained performance}} \\
        Peak accuracy (\%)
            & 72.74 & 72.80 & 73.11 & 75.89 \\
        Peak step
            & 30 & 20 & 20 & 20 \\
        Drop (pp)
            & 57.60 & 69.68 & 73.11 & 69.75 \\
        \bottomrule
    \end{tabularx}
\end{table*}

\noindent \textbf{Intuitor with random response filtering.}
We retain Intuitor's continuous reward
$r_i^{\mathrm{SC}}=\bar{c}_i$ and randomly discard
$\rho\in\{10\%,20\%,40\%\}$ of the sampled responses from each
training batch. Only retained responses contribute to policy
optimization. Selection operates at the response level and does
not depend on response length, self-certainty, or correctness.
Each run uses a fixed filtering rate throughout training.

\noindent \textbf{Top-2 rewards without filtering.}
Within each prompt group, we rank responses by mean self-certainty
and assign binary rewards:
\begin{equation}
    r_i^{\mathrm{top2}}
    =
    \mathbf{1}\!\left[
        i\in\operatorname{Top2}
        \bigl(\{\bar{c}_j\}_{j=1}^{N}\bigr)
    \right].
    \label{eq:ablation_top2}
\end{equation}
Exactly two responses receive reward one and the remaining six
receive reward zero. All eight responses participate in
group-relative advantage computation and policy optimization,
including those with zero reward. This baseline tests reward
discretization while preserving the underlying mean
self-certainty ranking.

\subsection{Results and Training Diagnostics}
\label{app:ablation_results}
Table~\ref{tab:ablation_validation} shows substantial late-stage
degradation in all four configurations.

\noindent \textbf{Random filtering does not prevent collapse.}
The three filtering rates yield similar peak accuracies,
ranging from 72.74\% to 73.11\%, but substantially different
final accuracies. Discarding 10\%, 20\%, and 40\% of responses
results in final accuracies of 0.00\%, 3.12\%, and 15.14\%,
with corresponding drops of 73.11, 69.68, and 57.60 percentage
points. The 10\% variant reaches zero accuracy at step 90 and
remains there through step 115. Increasing the filtering rate
to 40\% shifts the peak from step 20 to step 30 and leaves
more performance at the endpoint, but accuracy still falls
from 72.74\% to 15.14\%. Thus, stronger random filtering
moderates the observed deterioration without sustaining
the early gains.

\noindent \textbf{Reward discretization does not preserve performance.}
Top-2 achieves the highest peak among these configurations,
reaching 75.89\% at step 20. However, accuracy declines to
64.69\% at step 60, 21.83\% at step 80, and 6.14\% at
the final checkpoint, a drop of 69.75 percentage points.
Binary rewards therefore remain susceptible to collapse
when positive feedback is determined solely by the
mean self-certainty ranking.

\begin{figure*}
\vspace{-4pt}
    \centering
\includegraphics[width=\linewidth]{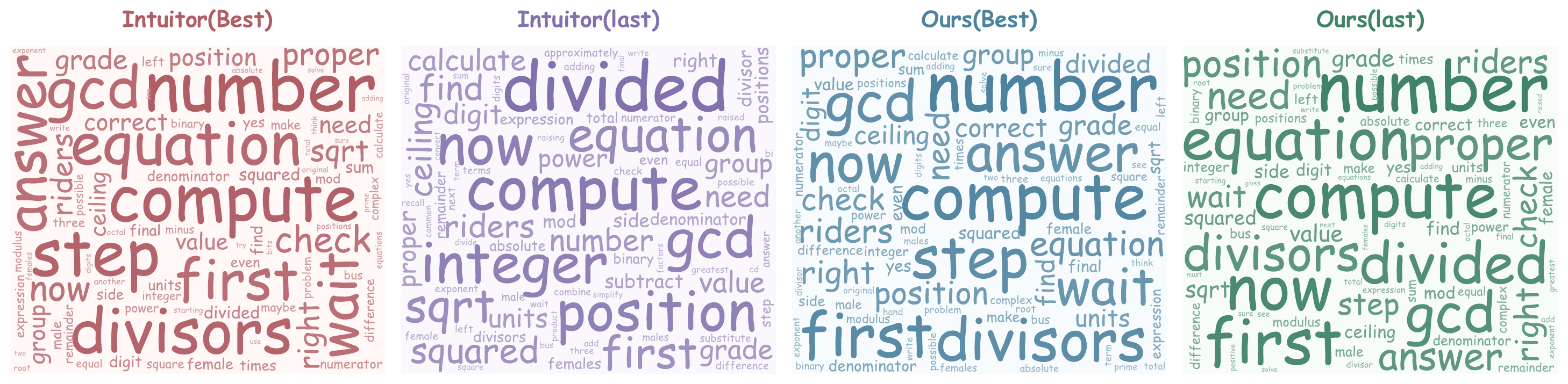}
\vspace{-24pt}
    \caption{Word clouds of generated responses from the best and last checkpoints of Intuitor and our method on the same ten MATH500 problems. Larger words indicate higher average within-response frequencies within each panel; common stopwords are omitted.}
    \vspace{-6pt}
\label{fig:wordcloud}
\end{figure*}

\subsection{Interpretation and Scope}
\label{app:ablation_scope}
The similar peaks but divergent endpoints distinguish initial
learning from sustained performance retention. Random filtering
reduces the number of responses used for optimization without
correcting the preference encoded by the reward. It provides
no criterion for selectively excluding responses whose
self-certainty is uninformative about correctness. Likewise,
top-2 discretization removes reward-magnitude differences but
continues to favor the same highest-ranked response means.
Neither intervention addresses the information lost when an
entire certainty profile is reduced to a single scalar.
These findings motivate SIF's comparisons across ranked
certainty windows and multiple peers, together with selective
policy updates.

These experiments establish the insufficiency of the tested
fixed-rate filtering and top-2 reward schemes in the observed
runs; they do not isolate the contribution of each SIF component.
Each configuration was evaluated once, and equal rollout
steps do not imply equal numbers of responses contributing
to optimization. The fixed-rate experiments also differ from
the control in Table~\ref{tab:gating_comparison}, which matches
SIF's filtering rate at each training step. Together, these
comparisons support the importance of how feedback and training
responses are selected, beyond simply reducing batch size or
binarizing mean self-certainty.

\section{Baseline Reproduction Configurations}
\label{app:baseline_configs}

\noindent \textbf {Shared training setup.}
We implement GRPO, Intuitor, Vigor, TTRL, and SCOPE within the \texttt{verl} framework~\cite{Sheng_2025}, using Llama-3.2-3B-Instruct, Qwen2.5-Math-1.5B, and Qwen3-1.7B as backbones. These runs use the MATH training split for policy optimization and the MATH test split for periodic validation, followed by evaluation on the reasoning benchmarks. Table~\ref{tab:baseline_shared_config} summarizes the common optimization settings; method-specific and run-specific differences are described below.

\begin{table}[t]
    \centering
    \small
    \caption{Common baseline training settings. The configured PPO mini-batch size is expressed in prompts and expands to 512 sampled responses before distribution across GPUs.}
    \label{tab:baseline_shared_config}
    \setlength{\tabcolsep}{4pt}
    \renewcommand{\arraystretch}{1.06}
    \begin{tabularx}{\columnwidth}{
        @{}>{\raggedright\arraybackslash}X
        >{\raggedright\arraybackslash}p{0.43\columnwidth}@{}
    }
        \toprule
        \textbf{Parameter} & \textbf{Value} \\
        \midrule
        Training budget & 1 epoch \\
        Prompt batch size & 64 \\
        Responses per prompt & 8 \\
        PPO mini-batch size & 64 \\
        Micro-batch per GPU & 4 \\
        PPO epochs per batch & 1 \\
        Optimizer & AdamW \\
        Adam coefficients & $(0.9, 0.999)$ \\
        Peak learning rate & $3\times10^{-6}$ \\
        Learning-rate schedule & Cosine decay \\
        Warmup ratio & 0.1 \\
        Weight decay & 0.01 \\
        PPO clipping range & $[0.8,1.2]$ \\
        Gradient clipping norm & 1.0 \\
        KL loss coefficient & 0.005 \\
        KL estimator & \texttt{low\_var\_kl} \\
        KL penalty in reward & Disabled \\
        Entropy coefficient & 0 \\
        Maximum response length & 3072 tokens \\
        Rollout temperature & 1.0 \\
        Rollout top-$p$ / top-$k$ & 1.0 / disabled \\
        Parameter updates & Full-model fine-tuning \\
        Gradient checkpointing & Enabled \\
        Distributed strategy & FSDP \\
        Generation backend & vLLM, bfloat16 \\
        Hardware allocation & 1 node, 8 A100 40G GPUs \\
        Validation/save interval & Every 10 steps \\
        \bottomrule
    \end{tabularx}
\end{table}

\noindent \textbf {GRPO.}
We use \texttt{adv\_estimator=grpo} and binary outcome rewards obtained by comparing extracted answers with MATH reference answers. Rewards are centered and divided by the sample standard deviation within each eight-response group, with a numerical stabilizer of $10^{-6}$. Policy optimization uses the clipped objective and KL regularization specified in Table~\ref{tab:baseline_shared_config}, without a learned critic or reward model.

\noindent \textbf {Intuitor.}
We use \texttt{adv\_estimator=Intuitor}. Each response receives its mean token-level self-certainty as the reward, followed by the same group normalization as GRPO. Reference-answer correctness is recorded for evaluation and diagnostics but does not determine the intrinsic reward. Intuitor uses all valid response tokens without SIF's minimum-length eligibility rule or voting-based reward construction.

\noindent \textbf {Vigor.}
We enable \texttt{use\_Vigor=True}, which selects the dedicated Vigor reward and advantage implementation. For each response, intrinsic score is
\begin{equation}
    r_i^{\mathrm{Vigor}}
    =-\sqrt{|o_i|}\,
      \bigl\|\nabla_\theta\ell_i^{\mathrm{NLL}}\bigr\|_2,
\end{equation}
where $\ell_i^{\mathrm{NLL}}$ is the mean negative log-likelihood over valid response tokens. The norm is computed over all trainable model parameters separately for each response, with FP32 norm accumulation. Scores are ranked within each prompt group and mapped linearly to $[-1,1]$. The resulting rewards are normalized using the sample standard deviation and a stabilizer of $10^{-4}$. Scoring gradients are discarded before policy optimization. Neither self-certainty nor reference-answer correctness enters this reward.

\noindent \textbf {TTRL.}
We enable \texttt{ttrl.enable=True} and set both \texttt{n\_votes\_per\_prompt} and \texttt{n\_samples\_per\_prompt} to eight. Extracted and normalized final answers from the eight sampled responses determine a majority pseudo-label. All eight responses are retained for training and scored against that pseudo-label using the migrated TTRL answer grader. The resulting outcome rewards use GRPO advantages with the $10^{-6}$ stabilizer.

\noindent \textbf {SCOPE.}
We enable \texttt{scope.enable=True} and \texttt{scope\_confidence=True}, with eight generated responses, eight voting candidates, four voters per pseudo-label subgroup, and $\lambda=0.7$. Step-wise confidence is computed from rollout log-probabilities using the migrated SCOPE procedure, with steps separated by the tokenizer's representation of a blank line. Candidate subgroup sizes are $\{1,2,4,8\}$. For each subgroup, four voters are sampled without replacement from the full eight-response pool. Confidence-weighted votes produce subgroup pseudo-labels, and the subgroup size is selected using the reward--diversity Pareto criterion. Voting is repeated after selection to construct the final labels. GRPO advantage normalization remains over full eight-response prompt group, rather than separately within each subgroup. Original reference answers are used for validation and diagnostics, while optimization uses pseudo-label rewards.

\section{Case Study}
To illustrate more intuitively the performance collapse observed during Intuitor's training, we selected the best checkpoint saved during training (before collapse) and the final checkpoint (after collapse) to generate responses to the same problem. For this comparison, we increased the maximum generation length to 8,192 tokens. A specific example is shown in Figure~\ref{fig:appendix-checkpoint-comparison}.

\begin{figure*}[t]
\begingroup
\setlength{\fboxsep}{5pt}
\setlength{\parindent}{0pt}
\setlength{\parskip}{0pt}
\setlength{\smallskipamount}{2pt}
\setlength{\medskipamount}{4pt}

% 本图局部设置：等宽字体、紧凑间距
\providecommand{\responsebox}[2]{}
\renewcommand{\responsebox}[2]{%
  \begin{minipage}[t]{0.485\linewidth}
    \vspace{0pt}
    \fcolorbox{black!35}{black!2}{%
      \begin{minipage}[t]
        {\dimexpr\linewidth-2\fboxsep-2\fboxrule\relax}
        \fontsize{8}{9}\selectfont\ttfamily\raggedright
        \setlength{\abovedisplayskip}{3pt}
        \setlength{\belowdisplayskip}{3pt}
        \setlength{\abovedisplayshortskip}{2pt}
        \setlength{\belowdisplayshortskip}{2pt}
        \textbf{#1}\par\smallskip
        #2
      \end{minipage}%
    }
  \end{minipage}%
}

\centering
\begin{minipage}{\textwidth}
\footnotesize\raggedright

\textbf{Problem.}
A proper divisor of an integer \(N\) is a positive divisor
of \(N\) that is less than \(N\).
What is the sum of the proper divisors of the sum of the
proper divisors of 284?

\par\smallskip
\textbf{Reference answer:} \(284\).
\par\medskip

% ==================================================
% 上排：我们的方法
% ==================================================
\textbf{Ours}\par\smallskip

% 左上：best，sample 02，seed 1235
\responsebox{Best (Step 40) --- Correct}{%
  \textbf{Full response:
    \textcolor{blue}{1,870} tokens.}\par
  \textbf{Mean self-certainty:
    \textcolor{blue}{92.07}.}\par\smallskip

  \textit{[...]}\par\smallskip

  So the proper divisors are 1, 2, 4, 71, \textbf{142}.
  \par\smallskip

  \textit{[...]}\par\smallskip

  Calculating that: 1 + 2 is 3, plus 4 is 7,
  plus 71 is 78, plus 142 is 220.
  So the sum of the proper divisors of 284 is 220.
  \par\smallskip

  \textit{[...]}\par\smallskip

  But proper divisors are those less than 220.
  So exclude 220 itself.
  So the proper divisors are
  1, 2, 4, 5, 10, 20, 11, 22, 44, 55, 110.
  \par\smallskip

  \textit{[...]}\par\smallskip

  So the sum of the proper divisors of 220 is 284.
  \par\smallskip

  Therefore, the answer is 284.
  \par\smallskip

  \textit{[...]}\par\smallskip

  \(\boxed{284}\)
}%
\hfill
% 右上：last，sample 06，seed 1239
\responsebox{Last (Step 115) --- Correct}{%
  \textbf{Full response:
    \textcolor{blue}{2,093} tokens.}\par
  \textbf{Mean self-certainty:
    \textcolor{blue}{94.81}.}\par\smallskip

  \textit{[...]}\par\smallskip

  So the divisors are 1, 2, 4, 71, 142, 284.
  Now, proper divisors are those less than 284.
  So excluding 284 itself, the proper divisors are
  1, 2, 4, 71, \textbf{142}.
  \par\smallskip

  \textit{[...]}\par\smallskip

  1 + 2 is 3. 3 + 4 is 7.
  7 + 71 is 78. 78 + 142 is 220.
  So the sum of the proper divisors of 284 is 220.
  \par\smallskip

  \textit{[...]}\par\smallskip

  Therefore, the proper divisors are
  1, 2, 4, 5, 10, 20, 11, 22, 44, 55, 110.
  \par\smallskip

  \textit{[...]}\par\smallskip

  So the sum of the proper divisors of 220 is 284.
  \par\smallskip

  Therefore, the answer is 284.
  \par\smallskip

  \textit{[...]}\par\smallskip

  \(\boxed{284}\)
}

\par\medskip

% ==================================================
% 下排：原版 Intuitor
% ==================================================
\textbf{Intuitor}\par\smallskip

% 左下：原 step20，sample 05，seed 1238
\responsebox{Best (Step 20) --- Correct}{%
  \textbf{Full response:
    \textcolor{blue}{2,691} tokens.}\par
  \textbf{Mean self-certainty:
    \textcolor{blue}{93.74}.}\par\smallskip

  \textit{[...]}\par\smallskip

  So the divisors of 284 are 1, 2, 4, 71, 142, 284.
  \par\smallskip

  But since we need proper divisors, which are less than 284,
  we exclude 284. So the proper divisors are
  1, 2, 4, 71, \textbf{142}.
  \par\smallskip

  \textit{[...]}\par\smallskip

  1 + 2 is 3, plus 4 is 7, plus 71 is 78,
  plus 142 is 220.
  \par\smallskip

  So the sum of the proper divisors of 284 is 220.
  \par\smallskip

  \textit{[...]}\par\smallskip

  Now, proper divisors are those less than 220,
  so we exclude 220. Therefore, the proper divisors are
  1, 2, 4, 5, 10, 11, 20, 22, 44, 55, 110.
  \par\smallskip

  \textit{[...]}\par\smallskip

  So the sum of the proper divisors of 220 is 284.
  \par\smallskip

  Therefore, the answer is 284.
  \par\smallskip

  \textit{[...]}\par\smallskip

  The sum of the proper divisors of the sum of the
  proper divisors of 284 is \(\boxed{284}\).
}%
\hfill
% 右下：原 step115，sample 05，seed 1238
\responsebox{Last (Step 115) --- Incorrect}{%
  \textbf{Full response:
    \textcolor{red}{3,641} tokens.}\par
  \textbf{Mean self-certainty:
    \textcolor{red}{101.60}.}\par\smallskip

  \textit{[...]}\par\smallskip

  284 divided by 1 is 284, so 1 is a divisor.
  Then 284 divided by 2 is 142, so 2 is a divisor.
  \par\smallskip

  \textit{[...]}\par\smallskip

  From the above, we saw that 1, 2, 4, and then 71?
  \textbf{Wait,} \(284\div4\) is 71. So 4 and 71.
  \textbf{Then check if 71 is a divisor.}
  \(71\times4\) is 284. So 71 is a divisor.
  \textbf{Then check if there are any others.}
  Let's check 71:
  \(71\times1=71\), \(71\times2=142\),
  \(71\times3=213\), \(71\times4=284\).

  So the divisors are 1, 2, 4, 71, and 284.
  But 284 is not a proper divisor,
  so proper divisors are 1, 2, 4, 71.
  \textbf{Let me confirm:}
  \[
    \begin{aligned}
      1\times284 &=284\\
      2\times142 &=284\\
      4\times71 &=284\\
      71\times4 &=284
    \end{aligned}
  \]

  So the divisors are 1, 2, 4, 71, and 284.
  Therefore, proper divisors are 1, 2, 4, 71.
  \par\smallskip

  \textit{[...]}\par\smallskip

  So the sum of the proper divisors of 284 is 78.
  \par\smallskip

  \textit{[...]}\par\smallskip

  So the proper divisors of 78 are
  1, 2, 3, 6, 13, 26, 39.
  \textbf{Let me check:}
  \(1\times78=78\), \(2\times39=78\),
  \(3\times26=78\), \(6\times13=78\),
  \(13\times6=78\), \(26\times3=78\),
  \(39\times2=78\).
  So the proper divisors are 1, 2, 3, 6, 13, 26, 39.
  \par\smallskip

  \textit{[...]}\par\smallskip

  Final answer: \(\boxed{90}\)
}

\end{minipage}
\par
\endgroup

\caption{%
Qualitative comparison on the same MATH500 problem.
Top: Ours, best checkpoint (step 40)
and last checkpoint (step 115).
Bottom: Intuitor, checkpoints at steps 20 and 115.
Both responses from our method are correct, whereas
Intuitor's last response produces verbose and repetitive
reasoning and gives an incorrect answer.
Authentic excerpts are shown; [...] marks omissions.
Bold emphasis, colors, and correctness labels are added.
Token counts and mean self-certainty refer to the full responses,
with temperature \(T=0.6\).
}
\label{fig:appendix-checkpoint-comparison}
\end{figure*}

\end{document}